\documentclass[12pt]{article} 
\usepackage{amsmath,amssymb} 
\usepackage{graphicx}
\usepackage{color}
\usepackage{amsfonts}
\usepackage{epstopdf}
\usepackage{algorithm}
\usepackage{algorithmicx}
\usepackage{algpseudocode} 

\usepackage{amsthm}
\newtheorem{lemma}{Lemma}

\newcommand{\norm}[1]{\left\|#1\right\|}
\newcommand\numberthis{\addtocounter{equation}{1}\tag{\theequation}}

\graphicspath{ {./pic/} }

\begin{document}

\title{AMS-Net: Adaptive Multiscale Sparse Neural Network with Interpretable Basis Expansion for Multiphase Flow Problems}


 
 \author{Yating Wang\thanks{Department of Mathematics, Purdue University, West Lafayette, IN 47907, USA.}
 \and Wing Tat Leung \thanks{Department of Mathematics, University of California, Irvine, Irvine, CA 92697, USA.}
 	\and Guang Lin\thanks{Department of Mathematics, School of Mechanical Engineering, Department of Statistics (Courtesy), Department of Earth, Atmospheric, and Planetary Sciences (Courtesy), Purdue University, West Lafayette, IN 47907, USA }} 
 
 \maketitle
 

\begin{abstract}

In this work, we propose an adaptive sparse learning algorithm that can be applied to learn the physical processes and obtain a sparse representation of the solution given a large snapshot space. 
Assume that there is a rich class of precomputed basis functions that can be used to approximate the quantity of interest. For instance, in the simulation of multiscale flow system, one can adopt mixed multiscale methods to compute velocity bases from local problems and apply the proper orthogonal decomposition (POD) method to construct bases for the saturation equation. We then design a neural network architecture to learn the coefficients of solutions in the spaces which are spanned by these basis functions. The information of the basis functions are incorporated in the loss function, which minimizes the differences between the downscaled reduced order solutions and reference solutions at multiple time steps. The network contains multiple submodules and the solutions at different time steps can be learned simultaneously. We propose some strategies in the learning framework to identify important degrees of freedom. To find a sparse solution representation, a soft thresholding operator is applied to enforce the sparsity of the output coefficient vectors of the neural network. To avoid over-simplification and enrich the approximation space, some degrees of freedom can be added back to the system through a greedy algorithm. In both scenarios, that is, removing and adding degrees of freedoms, the corresponding network connections are pruned or reactivated guided by the magnitude of the solution coefficients obtained from the network outputs. The proposed adaptive learning process are applied to some toy case examples to demonstrate that it can achieve a good basis selection and accurate approximation. More numerical tests are successfully performed on two-phase multiscale flow problems to show the capability and interpretability of the proposed method on complicated applications.
	
\end{abstract}

\section{Introduction}

Dynamical systems of flow and transport process in heterogeneous media are naturally existing in diverse science and engineering applications, such as groundwater flow, reservoir management, and so on. These physical problems are usually formulated in domains containing multiple scales, such as fractures at multiple length scales, or pores ranging from centimeters to meters. Numerical simulations for these problems are challenging since recovering all scale information will result in heavy computational burden. Furthermore, due to the lack of finest scale information, there are usually uncertainties in the computational model. It is necessary to develop model reduction techniques \cite{ee03,hfmq98, henning2009heterogeneous, chinesta2011short, aarnes2008mixed,arbogast2007multiscale} and construct fast alternatives to perform efficient simulations. The reduced order model can represent the physical properties of the full problem and can speed up the computations for the forward problem, which eventually helps to quantify the uncertainties in the model.

\noindent There are many model reduction methods including local and global approaches and have achieved significant success in numerous applications. In the family of local approaches, one can formulate appropriate local problems on coarse grid regions, construct effective properties or local multiscale basis functions, and further develop global systems on the coarse grid level. For instance, numerical upscaling, multiscale methods, and generalized multiscale methods \cite{egw10, GMsFEM13, AdaptiveGMsFEM, MixedGMsFEM, chung2017constraint, NLMC}. For global approaches\cite{ehg04, nonlinear_AM2015, cegg13, yang2016PodDeim}, such as the proper orthogonal decomposition method, one computes snapshots by solving several global problems and performs spectral decomposition to select the dominant modes. It has been extensively applied for numerical simulations of dynamical systems but still encounters difficulties for nonlinear problems. The objective of our work is to propose a framework which combines advanced deep learning techniques and multiscale basis construction methods, to obtain multiscale solutions with a sparse representation in the snapshot space.

\noindent Deep learning has become a quite popular approach for numerical approximation of nonlinear differential equations in recent days. Applications include developing surrogate models based on the properties of classical numerical solvers, such as constructing a multiscale neural network based on hierarchical multigrid solvers and encoder-decoder neural networks for solutions of heterogeneous elliptic PDEs \cite{Fan2018MNNH, PCDL_nz, karumuri2020simulator,li2020variational, bhattacharya2020model}. Physics-informed neural networks \cite{PINN1,PINN2,lu2019deeponet, meng2020composite} were proposed to incorporate physical laws in the loss function and limited data to train the neural network, and then get approximations of solutions in the whole temporal-spatial domain. However, learning full fine-scale solutions is challenging due to the extremely large number of parameters in the neural networks, some algorithms were established to design sparse neural network models to learn flow dynamics with high dimensional stochastic input coefficients \cite{wang2020efficient, wang2020bayesian}. Some other approaches include learning coarse grid effective properties using the nonlocal multicontinuum upscaling method or coefficients in the proper orthogonal decomposition (POD) projections \cite{wang2018NLMC_DNN, pod_dl_SW}. Furthermore, a deep neural network combined with multiscale model reduction techniques was investigated \cite{wang2020reduced} where the forward operators of the flow problem were learned in a reduced way without using POD approaches. In these works, the coefficients in the reduced order model are designed to have physical meanings, thus learning these quantities can provide important physical information without downscaling the coarse scale solution vector. In a more general setting, the coefficients of the basis functions are not directly related to the quantities of interest, then one may incorporate the basis functions as prior information in the training process. Some prior task-dependent dictionaries are incorporated to PINN method \cite{peng2020accelerating}, and an algorithm is proposed to take advantage of the features provided by dictionaries and achieve faster convergence.

\noindent In this work, we are interested in the multiscale two-phase problem in subsurface flow applications, where the equations are nonlinear and time dependent. One can parameterize the nonlinear equations and compute suitable basis functions for a set of sample parameters. However, the formed dictionaries may be too large and only a sparse selection of the bases in the dictionaries are needed for the solution approximation. Some model reduction techniques such as reduced basis method or greedy algorithm \cite{buffa2012priori, nguyen2008multiscale, jiang2017model, he2019reduced, gerner2012certified,li2020data, zhang2015multiscale} has been applied to solve parameterized elliptic PDEs. We aim to design an adaptive sparse learning algorithm with the help of precomputed basis functions as the prior dictionary, and apply it to the coupled two-phase flow systems. To be specific, for the construction of the prior dictionary for the flow equation where the model coefficient has high-contrast multiscale features, we adopt the mixed generalized multiscale finite element method \cite{MixedGMsFEM, chan2016adaptive} for velocity multiscale basis construction.
Given a specific source configuration, we first solve the system at several time instances in a small time interval. The fine scale saturation profiles at these time instances are used to form relative permeabilities in the flow equation. With this parameterization of the permeability, one can solve appropriate local problems on the coarse regions to get the velocity basis. Combining the local solutions in all coarse regions for each permeability configuration, we obtain a dictionary of basis functions which can be used to approximate the solutions of flow equations for different source terms. As for the saturation solution, we again use the saturation solutions described before as snapshots, and perform POD on the snapshot space. The dictionary for saturation approximation consists of all POD bases. It can be used to seek for solutions in later time steps given different source terms. We will then design suitable neural networks to learn the coefficients of the solutions in the reduced order spaces which are spanned by the bases of the prior dictionaries. However, due to the predefined dictionaries provide high-dimensional spaces, and only parts of basis functions are needed in the solution representation. We aim to reduce the solution space by an adaptive sparse learning approach. The idea is to first adopt the soft-thresholding technique to enforce the sparsity of the coefficient vectors learned from the neural network. Next, the network connections are pruned according to the sparsity of the coefficient output. This procedure helps to get rid of the less important basis functions in the solution representation and simplify the connections in the network architecture. However, if too many basis functions are dropped, the approximation space will not be sufficient to produce good approximations. We will further add some bases back through a greedy algorithm to enrich the approximation space, and the corresponding network connections will be reactivated simultaneously. The number of bases one would like to include in the approximation space can be fixed in advance, or the accuracy of the approximation can be prespecified. By an adaptive learning process, we expect to achieve a good basis selection and accurate approximation. Moreover, in our network architecture, submodules are designed to approximate the map from input to the first time step, and from previous time steps to later time steps. The final network is the composition of several submodules, and we are learning the entire dynamics, i.e, the solutions at all time steps, simultaneously. The loss functions are designed to minimize the differences between the downscaled reduced order solutions and fine scale solutions over all time steps.

\noindent The main contributions of our work are:
\begin{itemize}
\item Network functionality. We design a neural network architecture with an adaptive sparse learning algorithm, which can be applied to learn the map in the physical problem from the source term to the expansion coefficients of the multiscale basis in the solutions at many time steps. It learns the dynamics and identifies the sparse patterns simultaneously. 

\item Adaptivity. The sparsity of the basis expansion coefficient is enforced via soft thresholding. It can remove a large number of less important degrees of freedom during the training. On the other hand, to improve the accuracy, we can add some overdropped degrees of freedom back adaptively based on a greedy process. We observe that our proposed method achieves better accuracy compared to the projection solutions computed using the basis selected from the standard greedy algorithm/POD algorithm in some cases. 

\item Interpretability. Besides the accuracy benefits, the proposed adaptive learning algorithm can discover an active set of bases and select important degrees of freedom for the quantities of interest. We show that the sparsity patterns of the network output are potentially interpretable in some applications. 

\end{itemize}

\noindent The paper is organized as follows. In Section \ref{sec:prelim}, we describe the preliminaries of the model problem. The main methodology is discussed in Section \ref{sec:method} where we show the detailed construction of dictionary and the main algorithm. The numerical tests are demonstrated in Section \ref{sec:numerical_ex} to illustrate the capability of the proposed network. The numerical results show the efficiency and accuracy of our method. A conclusion is presented in Section \ref{sec:conclusion}.

\section{Problem Setup}\label{sec:prelim}

We consider the problem 
\begin{equation}
    \mathcal{L} u = f
\end{equation}
where $\mathcal{L}$ is a nonlinear time-dependent differential operator which contains multiscale features.

 We would like to seek the solution in an $N$ dimensional space 
 $$V_H = \text{span} \{\phi_1, \phi_2, \cdots, \phi_N \},$$ where $\phi_i$-s are precomputed snapshot bases. Denote by $u^j$ the solution at time step $j$, and suppose it has a sparse representation in this space, that is
\begin{equation}\label{eq:u_rep}
    {u}^j = \sum_{i=1}^N c_i^j \phi_i
\end{equation}
where $\textbf{c}^j = [c_1^j, \cdots, c_N^j]^T$ are sparse vectors.

Let $\mathcal{N}(\cdot;\Theta)$ be a deep neural network parameterized by $\Theta$, with given different realizations of source term $f$, we aim to use $\mathcal{N}$ to approximate the physical process, and realize sparse learning at the same time, that is 
\begin{equation}
    \{u^n\} \approx \mathcal{N}(f;\Theta,\phi_1, \phi_2, \cdots, \phi_N). 
\end{equation}
for all time step $n$.

\section{Methodology} \label{sec:method}
In this section, we will first present preliminaries for the problem of interest and the snapshot basis construction methods. Then we introduce our DNN architecture and training algorithm.

\noindent We consider two-phase incompressible flow problem in heterogeneous porous media. The flows follow Darcy’s law, and we neglect the capillary pressure and gravity effects in the model. The flow equation can be written as
\begin{equation}\label{eq:vel_2ph}
\begin{aligned}
u   = - \lambda(S) \kappa \nabla p \quad \quad &\text{in}  \quad \Omega\\
\text{div}  (u) = r \quad \quad &\text{in}  \quad \Omega\\
u\cdot n = 0 \quad \quad &\text{on}  \quad \partial \Omega
\end{aligned}
\end{equation}
where $\kappa$ is the absolute permeability field. The total mobility
\[
\lambda(S) = \displaystyle{\frac{\kappa_{rw} (S)}{\mu_w} +  \frac{\kappa_{ro} (S)}{\mu_o}}
\]
and $\kappa_{rw}$, $\kappa_{ro}$ are the relative permeability, $\mu_w$ is the viscosity for water, $\mu_o$ is the viscosity of oil. In real applications, $\kappa_{rw}$, $\kappa_{ro}$ nonlinearly depends of $S$. 


With a simplified notation, we abbreviate $S_w$ to be $S$ for the water phase . The saturation equation of $S$ reads
\begin{equation}\label{eq:sat}
\displaystyle{ \frac{\partial S}{\partial t} + u \cdot \nabla f(S) = q}
\end{equation}
where $f(S) =\displaystyle{ \frac{\kappa_{rw}(S)/ \mu_w }{\kappa_{rw}(S)/\mu_w+\kappa_{ro}(S)/ \mu_o} }$, and $q$ is the source term.

The saturation solution can be computed using the finite volume method on the fine grid, and a backward Euler scheme can be used for the time discretization. For each fine block $T_i$, the solution $S_i$ at time step ${n+1}$ can be obtained by
\begin{equation} \label{eq:sat_2ph}
\displaystyle{ S_i^{n+1} = S_i^{n} + \frac{\text{d} t}{|T_i|}  [ -\sum_{e_j \in \partial K_i} F_{ij} (S^{n+1}) + f (S^{n+1}) {q}_i^-  +  {q}_i^+ ] }
\end{equation}
where $q_i^- = \min(0,q_i)$, $q_i^+= \max(0,q_i)$. $e_j$ is the face between fine block $T_i$ and $T_j$. Denote by $u_{ij}$ the velocity on the face $e_j$, then $F_{ij}$ is
\begin{equation} \label{eq:upwind_flux_2ph}
F_{ij}(S^{n+1}) = \left\{
\begin{array}{ll}
\int_{e_j} (u_{ij}^{n+1} \cdot n) f_w(S_i^{n+1}) \quad \text{ if }  \quad u_{ij}^{n+1} \cdot n  \geq  0\\
\int_{e_j} (u_{ij}^{n+1} \cdot n)  f_w(S_j^{n+1}) \quad\text{ if } \quad u_{ij}^{n+1} \cdot n  <  0
\end{array}
\right.
\end{equation}
which is the upwinding flux.

\subsection{Dictionary construction}

We will consider two cases (1) learning the velocity dynamics, (2) learning the saturation dynamics, separately. In the first case, we will apply the mixed GMsFEM method \cite{MixedGMsFEM} to construct the local velocity basis. In the second case, we can use POD to perform global model reduction to obtain snapshots for saturation approximation. Those basis functions constitute the dictionary and will be used in the corresponding training tasks. 

\subsubsection{Local model reduction: Mixed GMsFEM basis for velocity}\label{sec:vel_basis}

Let $S_t$ be the saturation profiles at a few time instances $t = 1, \cdots, T$, obtained by solving the problem with a specific configuration of $f$. Then for each $S_t$, we compute the mobilities $\lambda(S_t)$ and use
\begin{equation*}
    \tilde{\kappa}_t : = \lambda(S_t)\kappa
\end{equation*} as different permeability profiles to compute basis functions. 

\noindent Denote by $\mathcal{T}_H$ the coarse grid of the computational domain $\Omega$. Let $\mathcal{E}^{H}$ be the set containing all coarse scale edges on the grid. In mixed GMsFEM, define the local region $\boldsymbol{\omega}_{i}$ as 
\begin{equation*}
      {\omega}_{i} = 
     \begin{cases}
      K_i^+ \cup K_i^-  & \text{ if } E_i\in \mathcal{E}^{H} \backslash \partial \Omega \\
      K_i  & \text{if } E_i \in \partial \Omega
    \end{cases} 
\end{equation*}%
which is a union of two coarse grid blocks sharing the edge $E_i$, with $i = 1, \cdots, N_e$, and $N_{e}$ is the total number of coarse edges. The basis functions for the velocity fields are constructed for each ${\omega}_{i}$.

\noindent To begin with, one constructs the snapshot space by solving local problems with a set of boundary conditions on ${\omega}_{i}$ associated with a coarse edge. The normal traces of each basis with respect to the coarse edge are resolved up to the fine level. Specifically, denote by $E_{i} = \bigcup^{L_i}_{j=1} e_{j}$, where $e_{j}$ is a fine edge on $E_i$, $L_{i}$ is the number of fine edges on $E_i$.
For each $j=1,2,\cdots, L_i$, we seek for local solutions $\psi_j^{\omega_i}$ by solving
\begin{equation*}
\begin{aligned}
    \tilde{\kappa}_t^{-1} \psi_{j,t}^{\omega_i} + \nabla p_{j,t}^{\omega_i} &= 0 &&\mbox{in } \omega_{i}, \\
\mbox{div} (\psi_{j,t}^{\omega_i} ) &= \alpha_j^{\omega_i} &&\mbox{in } \omega_{i}  \\
\psi_{j,t}^{\omega_i}  \cdot \boldsymbol{n}_{i} &= \delta^{\omega_i}_{j}  &&\mbox{on }  \partial \omega_i 
\end{aligned}
\end{equation*}
where $\delta^{\omega_i}_{j}$ is defined by
\begin{equation*}
\delta^{\omega_i}_{j} =
\begin{cases}
1 &\mbox{on } e_{j}, \\
0 &\mbox{on } \partial \omega_{i}\backslash e_{j},
\end{cases}
\end{equation*}
with $\boldsymbol{n}_i$ is the outward normal unit vector on $E_i$. Moreover, $\alpha_{i,j}$ satisfies the compatibility condition 
\begin{equation*}
    \int_{\omega_i} \alpha_j^{\omega_i} = \int_{\partial \omega_{i}} \psi_j^{\omega_i}  \cdot n_{i}
\end{equation*}
At this point, we obtain the snapshot space 
\begin{equation*}
V_{\text{snap},t}^{\omega_i} = \text{span} \{  \psi_{j,t}^{\omega_i}, \;\; j = 1, \cdots, L_i \},
\end{equation*}
where $i = 1,\cdots, N_{\omega}, \;\; t = 1,\cdots, T$, and $N_{\omega}$ is the number of coarse edges in the computational domain.

Next, one needs to propose a local spectral problem to perform model reduction for each $V_{\text{snap},t}^{\omega_i}$. The problem is to find eigenvalues $\lambda$ and eigenfunctions $v \in V_{\text{snap},t}^{\omega_i}$ such that
\begin{equation*}
a_i(v, w) = \lambda s_i(v, w), \qquad \forall w \in V_{\text{snap},t}^{\omega_i},
\end{equation*}
where $a_i$ and $s_i$ are symmetric positive definite bilinear operators.
As shown in \cite{MixedGMsFEM}, we can let
\begin{equation}\label{eq:spectral}
\begin{aligned}
a_i(v, w) &= \int_{E_i} {\kappa}_t^{-1} (v\cdot n_i) (w\cdot n_i), \\
s_i(v, w) &= \int_{\omega_{i}} {\kappa}_t^{-1} v \cdot w +\text{div}(v)\text{div}( w).
\end{aligned}
\end{equation}
Denote by $(\lambda_{j,t}^{\omega_i}, \phi_{j,t}^{\omega_i})$ be the eigen-pairs solved from \eqref{eq:spectral}, where the eigenvalues are sorted in an ascending order. Then the first $l_i$ dominant modes are selected to form the offline space $V_{\text{off},t}^{\omega_i}$. Finally, we take the union of all $V_{\text{off},t}^{\omega_i}$ as our dictionary. That is,

\begin{equation} \label{eq:dic_velocity}
    \mathcal{D}_{\text{vel}} = \{\phi_{j,t}^{\omega_i}, \;\; j = 1,\cdots, l_i;\;\; i = 1,\cdots, N_{\omega}; \;\; t = 1,\cdots, T \}.
\end{equation}

\subsubsection{Global model reduction of the saturation equation: POD basis construction}
In another perspective, we would like to learn saturation profiles and consider velocity as some hidden variables. Again, given some specific configuration of the source term $f$, we solve the system \eqref{eq:vel_2ph}-\eqref{eq:sat_2ph} sequentially. Denote by $S_t$ be the saturation profiles at a few time instances $t = 1, \cdots, T_0 \leq T$. These functions $\Phi = [S_1, \cdots, S_{T_0}]$ form the snapshot space, and we will perform the proper orthogonal decomposition (POD) on it. To be specific, one performs SVD on $\Phi$,
\begin{equation*}
    \Phi = V \Lambda W^T 
\end{equation*}
where $\Lambda$ is a diagonal matrix with singular values of $\Phi$, $V$ and $W$ are the left and right singular matrices. Arrange the singular values in a decreasing order $\sigma_1 \geq \sigma_2 \geq\sigma_{T_0}$, one can then choose the corresponding first few singular vectors in $V$ which capture the important modes in the dynamic process. Let $\phi_j, \; j = 1, \cdots, m$ be the vectors we chose, the POD space for saturation is then 
\begin{equation}\label{eq:dic_saturation}
\begin{aligned}
    V_{\text{sat}} &= \text{span} \{\phi_j, \; j = 1, \cdots, m\}\\
    \mathcal{D}_{\text{sat}} &= \{\phi_j, \; j = 1, \cdots, m\}
\end{aligned}
\end{equation}
and we will use it as our dictionary for the approximation of saturation solutions. For a newly given configuration of the source term, one can seek $S^\text{red}_t \approx S_t$ in the POD space.

\subsection{Network ingredients}\label{sec:dnn_str}
In this section, we present the main ingredients in our network architecture.

\noindent \textbf{Inputs and labels}: We consider a two-dimensional input $f\in \mathbb{R}^{d \times d}$ which can be arbitrary source terms in the equation, and a set of labels $\left(y^1; \cdots; y^T \right)$, where $T$ is the total number of time steps, and $y^j \in \mathbb{R}^n$. Here the labels $y^j$ can be velocity fields or saturation profile at time step $j$. 

\noindent \textbf{Network outputs}: The output of the network is denoted by $\left(\boldsymbol{c}^1, \cdots, \boldsymbol{c}^T \right)$,  where each $\boldsymbol{c}^j=(c_1^j, \cdots c_N^j)^T$ is a solution coefficient vector at time step $j$, with $\boldsymbol{c}^j \in \mathbb{R}^N$.

\noindent \textbf{Network architecture}: For the neural network $\mathcal{N}$, we will divide it into $T$ submodules
\begin{equation}
    \mathcal{N} = \mathcal{N}_T \circ \mathcal{N}_{T-1} \circ \cdots  \mathcal{N}_{1}
\end{equation}

\noindent For the first submodule, we aim to learn a map $\mathcal{N}_{1}$ from input $f$ to $\boldsymbol{c}_1$. Let $m$ be the number of layers in $\mathcal{N}_{1}$, which consists of some convolutional layers, an average pooling layer, and fully connected layers, that is 
\begin{equation*}
    \mathcal{N}_{1}: = S_{\gamma_1} \circ L_{1,m} \circ \sigma \circ  L_{1,m-1} \circ  \sigma \circ  L_{1,1}.
\end{equation*}
Let $K_j$ be an appropriate pooling kernel or convolution kernel,
\begin{equation}
\begin{aligned}
    &L_{1,j} (\boldsymbol{x}) = K_j * \boldsymbol {x}, \quad j = 1, \cdots, m-2\\
    &L_{1,m-1}(\boldsymbol{x})  = \boldsymbol{W}_{1, m-1} \left( \text{vec} (\boldsymbol{x}) \right) + \boldsymbol {b}_{1, m-1} \\
    &L_{1,m} (\boldsymbol{x}) = \boldsymbol{W}_{1,m} \boldsymbol{x}.  
\end{aligned}
\label{eq:module_1}
\end{equation}
Moreover, we write the intermediate output from $\mathcal{N}_{1}(f)$ as $\boldsymbol{c}^1$. 

\noindent For the other submodules $\mathcal{N}_t, \quad t=2,\cdot, T$, we have
\begin{equation}\label{eq:N_t}
    \mathcal{N}_{t}: = S_{\gamma_t} \circ L_{t, m} \circ \sigma \circ  L_{t, m-1} \circ  \sigma \circ  L_{t, 1}
\end{equation}
where

\begin{equation}
\begin{aligned}
    &L_{t,j} (\boldsymbol{x})  = \boldsymbol{W}_{j, m-1} \boldsymbol{x}  + \boldsymbol {b}_{j, m-1} \\
    &L_{1,m} (\boldsymbol{x}) = \boldsymbol{W}_{1,m} \boldsymbol{x},   
\end{aligned}
\label{eq:module_t}
\end{equation}
 
\noindent Here, $\sigma$ is a nonlinear activation function, for example, leaky RELU, which is defined as
\begin{equation}
   \sigma (\boldsymbol{x}) = \begin{cases}
      x & \text{if $t>0$}\\
      \alpha x & \text{otherwise}
    \end{cases} 
\end{equation}
for some constant $\alpha \in (0,1)$.

\noindent $S_{\gamma_t}$ is a soft-thresholding function, defined as
\begin{equation}\label{eq:soft_thres}
     S_{\gamma_t}(\boldsymbol{x}) = \text{sign}(\boldsymbol{x}) (|\boldsymbol{x}|-\gamma_t)_+
\end{equation}
with some constant threshold $\gamma_t$. The soft-thresholding function will help us to enforce sparsity on the predicted solution coefficient vectors $\boldsymbol{c}$.

\noindent Similarly, we denote by $\boldsymbol{c}^t$ the intermediate output from $\mathcal{N}_{t}(f)$.

The architecture of the network can be illustrated as in Figure \ref{fig:network}.
\begin{figure}[!hbt]
	\centering
	\includegraphics[scale=0.6]{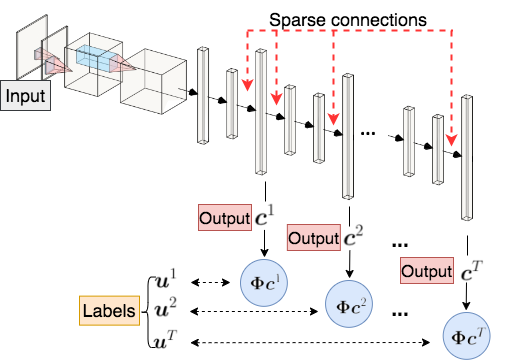}
	\caption{An illustration of the network architecture.The network's input is the realization of a random source, the outputs contain the prediction of expansion coefficients $\textbf{c}_j$ for multiscale basis in the solutions at multiple time steps ($j=1, \cdots, T$). The basis funtions $\boldsymbol{\Phi}$ are incorporated in the loss function.}
	\label{fig:network}
\end{figure}

\subsection{Loss function with basis functions}

Given a set of training pairs \\
$\{ f_k, \left(y^1_k, \cdot, y^n_k\right) \}$, our goal is then to find $\Theta^*$ for the network $\mathcal{N}(\cdot;\Theta)$ by solving an optimization problem
\begin{align*}
\Theta^* &= \text{argmin}_{\Theta} \mathcal{L}\left(\mathcal{N}(f;\Theta); \{y^j\}_{j=1}^T \right) \\
&= \text{argmin}_{\Theta} \frac{1}{K}\sum_{k=1}^{K} \sum_{j=1}^{T}  ||y^j_k - \Phi \boldsymbol{c}^j_k ||^2_2, \numberthis \label{eq:loss}
\end{align*}
where $\Phi = \left[ \phi_1 \quad  \phi_2 \quad  \cdots \quad  \phi_N \right]$ is the matrix formed by the precomputed bases. $K$ is the number of the samples, $N$ is the number of bases in the dictionary. We apply the preconditioned SGD to solve the optimization problem in \eqref{eq:loss}.

\subsection{Adaptive sparse learning algorithm}
In this section, we will propose our main algorithm, the adaptive sparse learning algorithm. The algorithm consists of two parts: removing basis/pruning connections (as illustrated in \ref{sec:sparse_coef} and \ref{sec:sparse_net}), and adding basis/reactivating connections (as illustrated in \ref{sec:add_basis}).

\subsubsection{To reduce model order with sparse output coefficients}\label{sec:sparse_coef}
To ensure the sparsity of the the model output, we apply the soft thresholding function \eqref{eq:soft_thres}, which can be further written as 
\begin{equation}
     S_{\gamma_t}(\boldsymbol{x}) =  \begin{cases}
      x-\gamma_t & \text{if $x\geq \gamma_t $}\\
      0 & \text{if $- \gamma_ <x< \gamma_t $}\\
      x+\gamma_t & \text{if $x\leq -\gamma_t $}. 
    \end{cases} 
\end{equation}
\noindent After the action of $S_{\gamma_t}$, we obtain the output coefficients $\boldsymbol{c}^t$. The soft thresholding function will cut off those coefficients with small magnitudes. Then the sparse coefficient vector will be multiplied by the normalized basis function matrix $\boldsymbol{\Phi}$. This procedure results in a removal of some unimportant basis functions during the training.

We remark that the soft-thresholding process is commonly involved in $l_{1}$ minimization algorithms. One of the ways to view this process is that the soft-thresholding function defines an active set in the optimization process. During the training, we observe that the gradient vector corresponding to the $i$th entry in the soft thresholding layer will vanish once the coefficient output in that layer is smaller than the thresholding parameter.

\subsubsection{Pruning network connections}\label{sec:sparse_net}
\noindent In our framework, we also want to enforce sparsity on the network connections based on the sparse pattern of intermediate network outputs $\boldsymbol{c}^t$. At this point, denote by $\Theta = \Theta_s \cup \Theta_d $ the network parameters, where 
\begin{equation*}
    \Theta_s = \{ \boldsymbol{W}_{t,m}, \boldsymbol{W}_{t+1,1}, \;\; \text{ for all } t = 1, \cdots, T-1 \}
\end{equation*}
corresponds to the parameters in the sparse layer, and $\Theta_d$ corresponds to the parameters in the rest of layers. The weight matrix $\boldsymbol{W}$ can be referred to \eqref{eq:module_1} and \eqref{eq:module_t}.

\noindent We would like to cut connections to and from $\boldsymbol{c}^t$, based on the magnitude of $\boldsymbol{c}^t_j$ during the training. To be specific, if at the current training epoch, the $\boldsymbol{c}^t_{j,k} \leq \gamma_t^{\text{remove}}$ (one can just take $\gamma_t^{\text{remove}} = 0$) for the $k$-th entry, then we will let
\begin{equation*}
   \boldsymbol{W}_{t,m}(k,:) = 0, \;\;  \boldsymbol{W}_{t+1,1}(:,k) = 0.
\end{equation*}

\subsubsection{Adaptively enrich the solution space}\label{sec:add_basis}

\noindent During the training, the sparse pruning procedure may overly drop some components, one can then adaptively add the basis back to the training by reactivating some previously pruned network connections. To illustrate the idea, denote by $y_{\text{true}} \in \mathbb{R}^{n\times B_n}$ the reference solutions, $y_{\text{pred}} \in \mathbb{R}^{n\times B_n}$ the corresponding predictions from the neural network, where $B_n$ is the number of samples in a batch. Let $P_u \in\mathbb{R}^{m_u \times n}$ be the matrix containing the unselected bases in the current stage, where $m_u$ is the number of bases left in the pool. One first computes the differences between the true and prediction solutions
\begin{equation*}
   R^t =  y^t_{\text{true}} -  y^t_{\text{pred}}.
\end{equation*}
One then computes the inner product of the error $R^t$ with the unselected basis
\begin{equation}\label{eq:Et}
   E^t = P_u R^t.
\end{equation}

\noindent Since a batch of samples is used when computing $R^t$, we need to compute the absolute mean of $E^t$ across these samples, denoted by $\bar{E^t}$.  We then sort the absolute mean $\bar{E^t}$ in a descending manner. Denote by $[i^t_1, i^t_2, \cdots, i^t_{m_u}]$ the sorted indices which specifies how the elements of the absolute mean of $E^t$ were rearranged.  Recall that the total number of bases in the dictionary is $N$, then the current number of selected bases is $N-m_u$.

\noindent There are two ways to add degrees of freedom back. 
\begin{itemize}
    \item Fix the number of bases. The first approach is to set a target $M$ as the total number of bases we want to include in the final solution representation. Then,  if $N-m_u < M$, we will select the first $m_c = M - (N-m_u)$ bases corresponds to $[i^t_1, i^t_2, \cdots, i^t_{m_c}]$, add them back in the solution representation, and remove them from the unselected pool $P_u$.
    \item Fix a threshold parameter. The second approach is to set the thresholding parameters $\delta_{\text{add}}$. That is, 
if $[\bar{E^t}]_{j} <\gamma_t^{\text{add}}$, then we let $m_c = j$ and add bases correspond to the columns with indices $[i^t_1, i^t_2, \cdots, i^t_{m_c}]$, and remove them from the pool $P_u$.
\end{itemize} 

\noindent If one fixes the target number of bases in advance, it can have a control on the number of bases selected during the training process and produce a desired dimension for the reduced order model. This works for the case when we have an approximate bound for the number of important bases. More generally, without the knowledge of the exact number of bases are needed, we usually want to control the accuracy of the approximation. Then we can use the thresholding parameter $\gamma_t^{\text{add}}$ to determine the number of bases to add in the process. The relationship between the threshold $\gamma_{t}^{add}$ and the error is presented by the following lemma.

\begin{lemma}
Let $\gamma_{t}^{add}$ be a given threshold for adding bases. Assume
the iteration process is converged to a static state where no basis
will be added to the system in the iteration process. For any unselected basis
$\phi\in\mathbb{R}^{d}$, we have 
\begin{equation}
\cfrac{1}{B_{n}} \sum_{i=1}^{B_n} \|y_{true,i}^{t}-y_{pred,i}^{t}\|_{2}-\cfrac{1}{B_{n}} \sum_{i=1}^{B_n} \min_{(c_{1},\dots c_{B_{n}})\in\mathbb{R}^{B_{n}}}\|y_{true,i}^{t}-y_{pred,i}^{t}- \phi c_{i} \|_{2}<\gamma_{t}^{add}
\end{equation}
where $y_{true,i}^{t}$ and $y_{pred,i}^{t}$ are the $i$-th
column of $y_{true}^{t}$ and $y_{pred}^{t}$, and $i=1,\cdots, B_n$, $B_n$ is the batch size. 
\end{lemma}

\begin{proof}
If no bases are added to the system in the iteration process, we have
\begin{equation}
[\bar{E}^{t}]_{j}<\gamma_{t}^{add}\;\;\;\;\; \forall j=1, \cdots, m_u; \;\; t = 1, \cdots, T.
\end{equation}
where $\bar{E}^{t}\in \mathbb{R}^{m_u}$ is the mean of $E^t$ over all samples as defined in \eqref{eq:Et}. Thus, for any unselected basis
$\phi\in\mathbb{R}^{d}$
\begin{equation}
\sum_{i=1}^{B_n}|\phi^{T}(y_{true,i}^{t}-y_{pred,i}^{t})|<\gamma_{t}^{add}.
\end{equation}
Since $\phi^{T}\phi=\|\phi\|_{2}^2=1$, we have 
\begin{equation*}
\begin{aligned}
&\|y_{ture,i}^{t}-y_{pred,i}^{t}\|_{2}-|\phi^{T}(y_{true,i}^{t}-y_{pred,i}^{t})|\\ 
& \leq\|y_{ture,i}^{t}-y_{pred,i}^{t}-\phi\phi^{T}(y_{ture,i}^{t}-y_{pred,i}^{t})\|_{2}\\
 & \leq\|y_{ture,i}^{t}-y_{pred,i}^{t}-\phi c_{i}\|_{2}
\end{aligned}
\end{equation*}
for any $c_{i}\in\mathbb{R}$. Therefore, we obtain 
\begin{equation*}
\begin{aligned}
& \cfrac{1}{B_{n}}\sum_{i=1}^{B_n} \Big(\|y_{ture,i}^{t}-y_{pred,i}^{t}\|_{2}-\|y_{ture,i}^{t}-y_{pred,i}^{t}-\phi c_{i}\|_{2}\Big)\\
& \leq\cfrac{1}{B_{n}}\sum_{i=1}^{B_n} |\phi^{T}(y_{true,i}^{t}-y_{pred,i}^{t}| <\gamma_{t}^{add}
\end{aligned}
\end{equation*}
for any $(c_{1},\dots c_{B_{n}})\in\mathbb{R}^{B_{n}}$. This completes the proof.
\end{proof}
By this lemma, we can see that using the threshold $\gamma_{t}^{add}$ to control the number of bases can give us a control on the error, in the sense that adding any unselected basis can only make a small improvement to the average $l_2$ error using the training data set.

\noindent Furthermore, the corresponding rows or columns in the weight matrix will be reactivated,
\begin{equation*}
   \boldsymbol{W}_{t,m}(j,:) =\boldsymbol{z}, \;\;  \boldsymbol{W}_{t+1,1}(:,j) = \boldsymbol{z},\;\;  j = i^t_1, i^t_2, \cdots, i^t_{m_c}
\end{equation*}
where $\boldsymbol{z}$ are vectors with i.i.d samples generated from uniform distributions.

\noindent The sparse procedures described above will be performed on an adaptive basis.
Denote by $\boldsymbol{\Theta}_0$ the initial model parameters, $n_b$ the number of burning in steps, $n_e$ the total number of epochs, $\eta$ the learning rate. Suppose we would like to update the sparsity information every other $nn$ steps. Let $\gamma_t^{\text{remove}}$ be the thresholding parameters for removing bases,  $\gamma_t^{\text{add}}$ be the thresholding parameter for adding bases. $I^t_u$ is the unselected bases indices, $I^t_c$ are the indices of bases we would like to add. $M$ is the target number of bases, $N$ is the total number of bases in the dictionary. The algorithm is summarized in Algorithm \ref{alg:sparse}.

\begin{algorithm}
\caption{Adaptive Multiscale Sparse Neural Network for Basis Expansion Learning }\label{alg:sparse}
    \begin{algorithmic}[1]
    \Procedure {AMS-Net}{$\boldsymbol{\Theta}_0$, $n_b$, $n_e$, $m$, $\gamma_t^{\text{remove}}$, $\gamma_t^{\text{add}}$} 
    \ForAll{$i = 1: n_e$}
        \If{$i > n_b$ and $ \mod (i, nn) = 0 $}
            \State $\boldsymbol{c}_t \gets \text{ output from } \mathcal{N}_t$, \text{ where } $\mathcal{N}_t$ \text{ is defined in } \eqref{eq:N_t}
            
            \State $I^t_u \gets  \;\; \{j\} \text{ such that }  \lvert {\boldsymbol{c}^t_j} \rvert  < \gamma_t^{\text{remove}} $
            \State $m_u \gets  \text{length}(I^t_u)$
            \State \text{Set weight parameters } $\big\{ \boldsymbol{W}_{t,m}(I^t_u,:), \; \boldsymbol{W}_{t+1,1}(:,I^t_u)\big\}$  \text{ to zero in } $\boldsymbol{\Theta}_i$ 
           	\State $\bar{E^t} \gets \frac{1}{B_n}\sum_{i=1}^{B_n} |P_u y^t_{\text{true},i} -  y^t_{\text{pred},i}|$
           	\State $I \gets $ descending ordered indices of $\bar{E^t}$
             	\If{ Set target number of bases $M$ and $N- m_u < M$ }
            		\State $m_c \gets M-(N - m_u)$ 
           		 \ElsIf{Set adding threshold $\gamma_t^{\text{add}}$ and $[\bar{E^t}]_{1}<\gamma_t^{\text{add}}$ }
           		 	\State $m_c \gets j$ where $[\bar{E^t}]_{j} < \gamma_t^{\text{add}}$
           	\EndIf 
           	\State $I^t_c \gets I(1:m_c)$
            \State \text{Reactivate weight parameters } $ \big\{\boldsymbol{W}_{t,m}(I^t_c,:), \; \boldsymbol{W}_{t+1,1}(:,I^t_c)\big \}$  \text{ from } $\boldsymbol{\Theta}_i$.
        \EndIf
        
         \State \text{Update } $\boldsymbol{\Theta}_{i+1}$ \text{ by stochastic gradient based algorithms} 
    \EndFor
    \State \textbf{return}  $\boldsymbol{\Theta}_i$
    \EndProcedure
    \end{algorithmic}
\end{algorithm}

\section{Numerical example} \label{sec:numerical_ex}

\subsection{Toy example}
We first test the sparse learning algorithm with a dictionary on a simple static example.
Consider the two-dimensional elliptic equation on $\Omega = [0,1]\times [0,1]$
\begin{equation*}
\begin{aligned}
    -\text{div}(k(u) \nabla u) &= f(x,y; \boldsymbol{\alpha})\\
    u &= 0 \text{ on } \partial \Omega 
\end{aligned}
\end{equation*}
where the ground truth for $u$ is
\begin{equation*}
\begin{aligned}
    u(x,y; \boldsymbol{\alpha}) &= \alpha_1 \sin(\pi x) \sin(\pi y) + \alpha_2 \sin(2\pi x) \sin(2\pi y) \\
    &+ \alpha_3  \sin(3\pi x) \sin(\pi y)+ \alpha_4  \sin(4\pi x) \sin(2\pi y)
\end{aligned}
\end{equation*}
and $\boldsymbol{\alpha} = [\alpha_1, \alpha_2, \alpha_3, \alpha_4]$, with $\alpha_i$ i.i.d samples generated from normal distribution. For $\alpha_1$, we have mean $1$ and standard deviation $2$. For $\alpha_2$, we have mean $0$ and standard deviation $3$. For $\alpha_3$, we have mean $-2$ and standard deviation $3$. For $\alpha_4$, we have mean $5$ and standard deviation $2$,

We create the dictionary $D_{1,2}$ with 
\begin{align*}
    D_1 &= \{ 1,\; \sin(\pi x),  \;\sin(2 \pi x),  \; \cdots,  \;\sin(M\pi x)\}\\
    D_2 &= \{ 1,\; \sin(\pi y),  \;\sin(2 \pi y),  \; \cdots,  \; \sin(M\pi y)\}\\
    D_{1,2} &= \{d_1 d_2  \; |\; d_1 \in D_1, \; d_2 \in D_2\}.
\end{align*}
where $M=9$, thus we have $100$ basis function in the set $D_{1,2}$.

\subsubsection{Linear case}\label{sec:toy_linear}
In the linear case, we take $k(u) = 1$, and
\begin{equation*}
\begin{aligned}
    f(x,y; \boldsymbol{\alpha}) &= 2\pi^2 \alpha_1 \sin(\pi x) \sin(\pi y) + 8\pi^2\alpha_2 \sin(2\pi x) \sin(2\pi y) \\
    &+ 10\pi^2\alpha_3 \sin(3\pi x) \sin(\pi y)+ 20\pi^2 \alpha_4 \sin(4\pi x) \sin(2\pi y)
\end{aligned}
\end{equation*}
We assume the dataset for training is small, and there are $100$ sample pairs. The network to be trained has only a submodule $\mathcal{N}_1$ as described in Section \ref{sec:dnn_str}. Among all samples, $80$ percents are used for training, and $20$ percents are used for testing. In this example, the bases included in the solutions are known exactly, so we only prune the network and remove some bases (the thresholding parameter is chosen to be $0.5$), and do not add basis functions. The purpose is to see whether the adaptive pruning can choose the correct set of basis in the end. The numerical results indicate that, applying the adaptive sparse algorithm, the network can identify the $4$ basis functions which constitute the ground truth $u$ automatically. The number of bases selected during training is presented in Figure \ref{fig:ex1_linear_basis}. In this example, we start the sparse pruning after 100 epochs and record the number of bases every 100 epochs. We can see that the number of bases drops very fast and the network can find the correct set of basis after 1700 epochs. 

\begin{figure}[!hbt]
	\centering
	\includegraphics[scale=0.3]{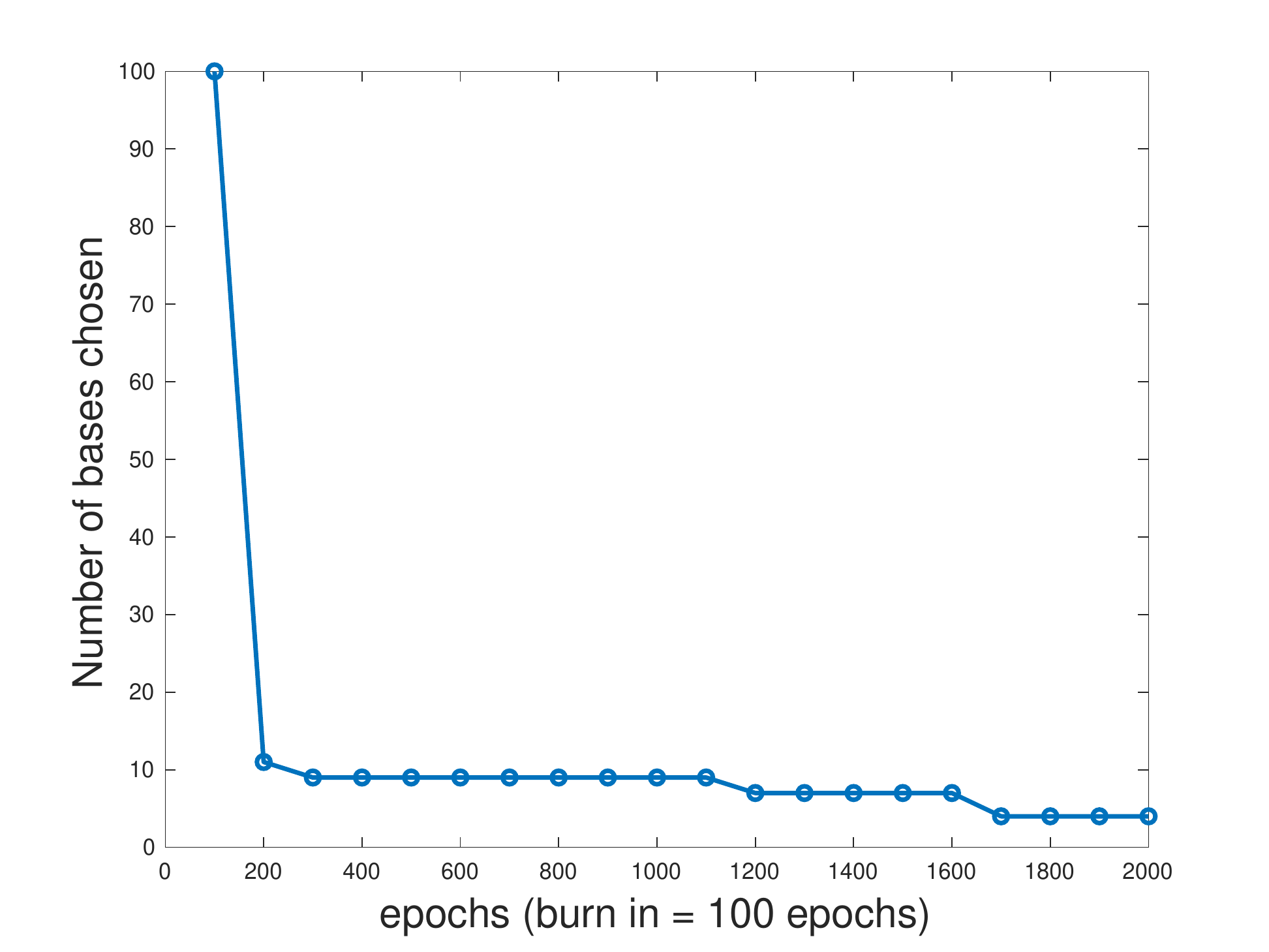}
	\caption{(Section \ref{sec:toy_linear}) Toy example, linear case, the number of selected bases v.s. training epochs. The number of training samples is 80. The total number of basis functions is 100. After 200 epochs, the sparsity reaches 90\%. The network identifies the correct sparsity pattern and finds the true basis functions in epoch 1700.}\label{fig:ex1_linear_basis}
\end{figure}

The training and testing history are presented in \ref{fig:ex1_linear}. We compare the case (1) when we do adaptive pruning during the training (``AMS-net (p)")and (2) do not prune the network during the training. It shows that our proposed adaptive pruning method can achieve faster training and produce better prediction results.

\begin{figure}[!hbt]
	\centering
	\includegraphics[scale=0.3]{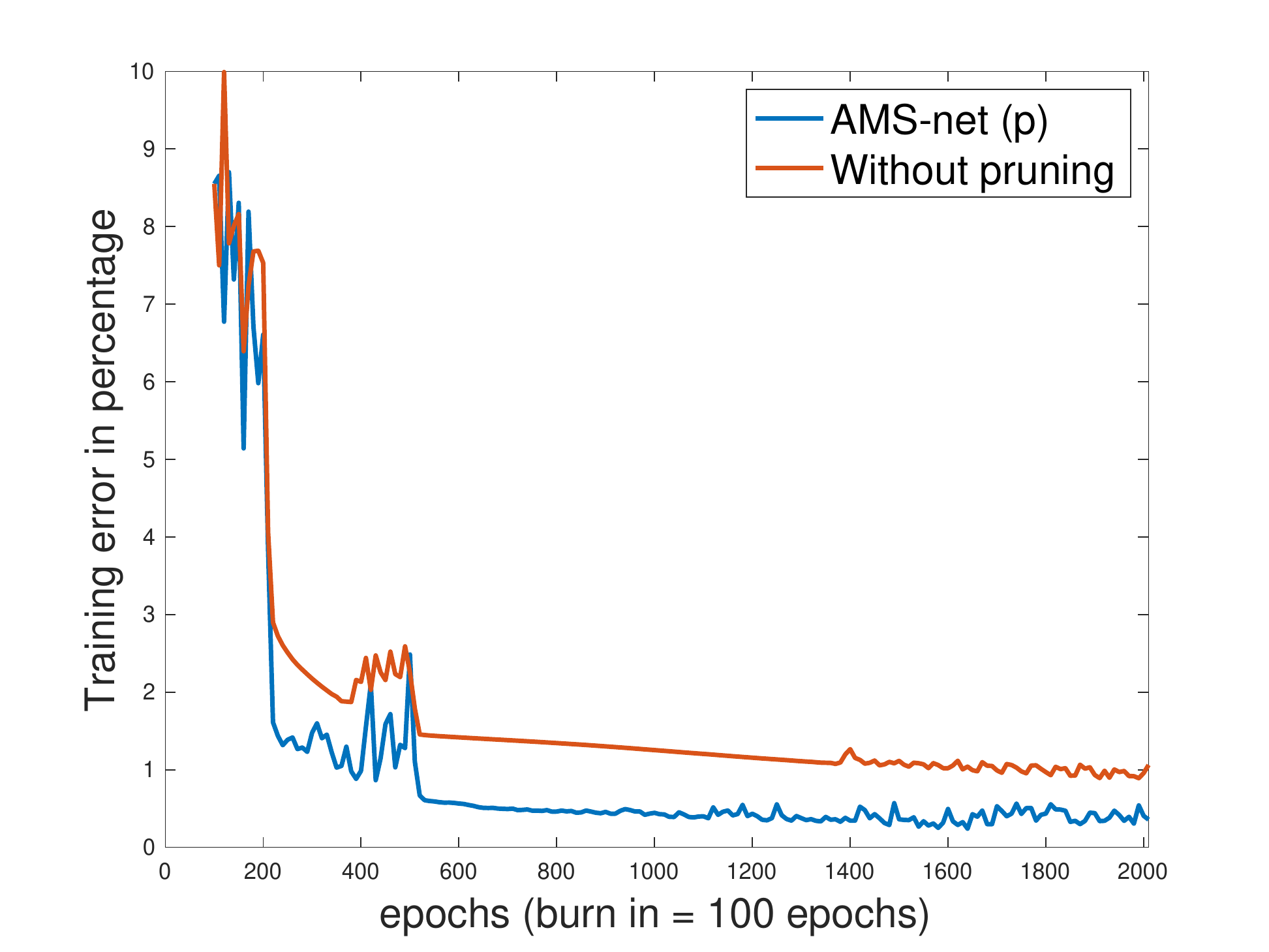}
	\includegraphics[scale=0.3]{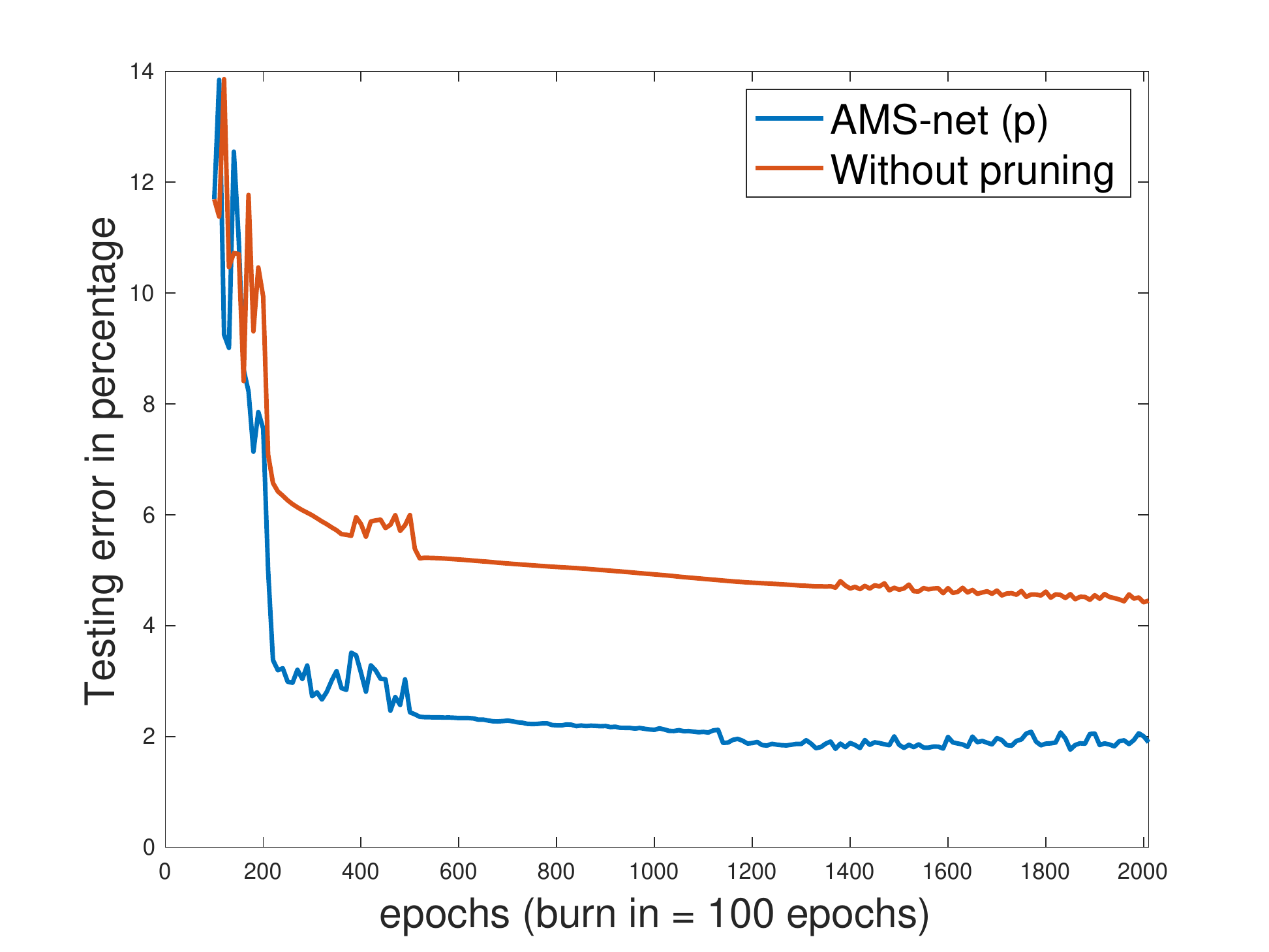}
	\caption{(Section \ref{sec:toy_linear}) Toy example, linear case. Left: training error history, right: testing error history. Comparison between the results using our proposed pruning algorithm and without pruning. The number of training/testing samples is 80/20. It shows that training with pruning is more efficient.}\label{fig:ex1_linear}
\end{figure}


\subsubsection{Nonlinear case}\label{sec:toy_nonlinear}
In the linear case, we take $k(u) = u$, and
\begin{equation*}
\begin{aligned}
    f(x,y; \boldsymbol{\alpha}) &= -( \alpha_1 \cos(\pi x) \sin(\pi y) + 2\pi \alpha_2 \cos(2\pi x) \sin(2\pi y) \\
    &+ 3\pi \alpha_3 \cos(3\pi x) \sin(\pi y)+ 4\pi \alpha_4 \cos(4\pi x) \sin(2\pi y) )^2 \\
        &- ( \alpha_1 \sin(\pi x) \cos(\pi y) + 2\pi \alpha_2 \sin(2\pi x) \cos(2\pi y) \\
    &+ 3\pi \alpha_3 \sin(3\pi x) \cos(\pi y)+ 4\pi \alpha_4 \cos(4\pi x) \sin(2\pi y) )^2\\
    & -( +\alpha_1 \sin(\pi x) \sin(\pi y) + \alpha_2 \sin(2\pi x) \sin(2\pi y) \\
    &+ \alpha_3  \sin(3\pi x) \sin(\pi y)+ \alpha_4  \sin(4\pi x) \sin(2\pi y))\\
    & -(2\pi^2 \alpha_1 \sin(\pi x) \sin(\pi y) + 8\pi^2\alpha_2 \sin(2\pi x) \sin(2\pi y) \\
    &+ 10\pi^2\alpha_3 \sin(3\pi x) \sin(\pi y)+ 20\pi^2 \alpha_4 \sin(4\pi x) \sin(2\pi y))
\end{aligned}
\end{equation*}

In this nonlinear case, we generate more sample pairs, $1000$ in total, to train the neural network. Among them, $800$ samples are used for training, and $200$ samples are used for testing. In this example, similar as before, we only perform pruning without adding. We choose the thresholding parameter to be $0.3$. The number of bases selected during training is presented in Figure \ref{fig:ex1_nonlinear_basis}. Here we record the number of bases every 20 epochs. Again, we observe that the number of bases drops very fast in the beginning of our algorithm, and the network can identify the correct basis sets after 760 epochs. 

\begin{figure}[!hbt]
	\centering
	\includegraphics[scale=0.3]{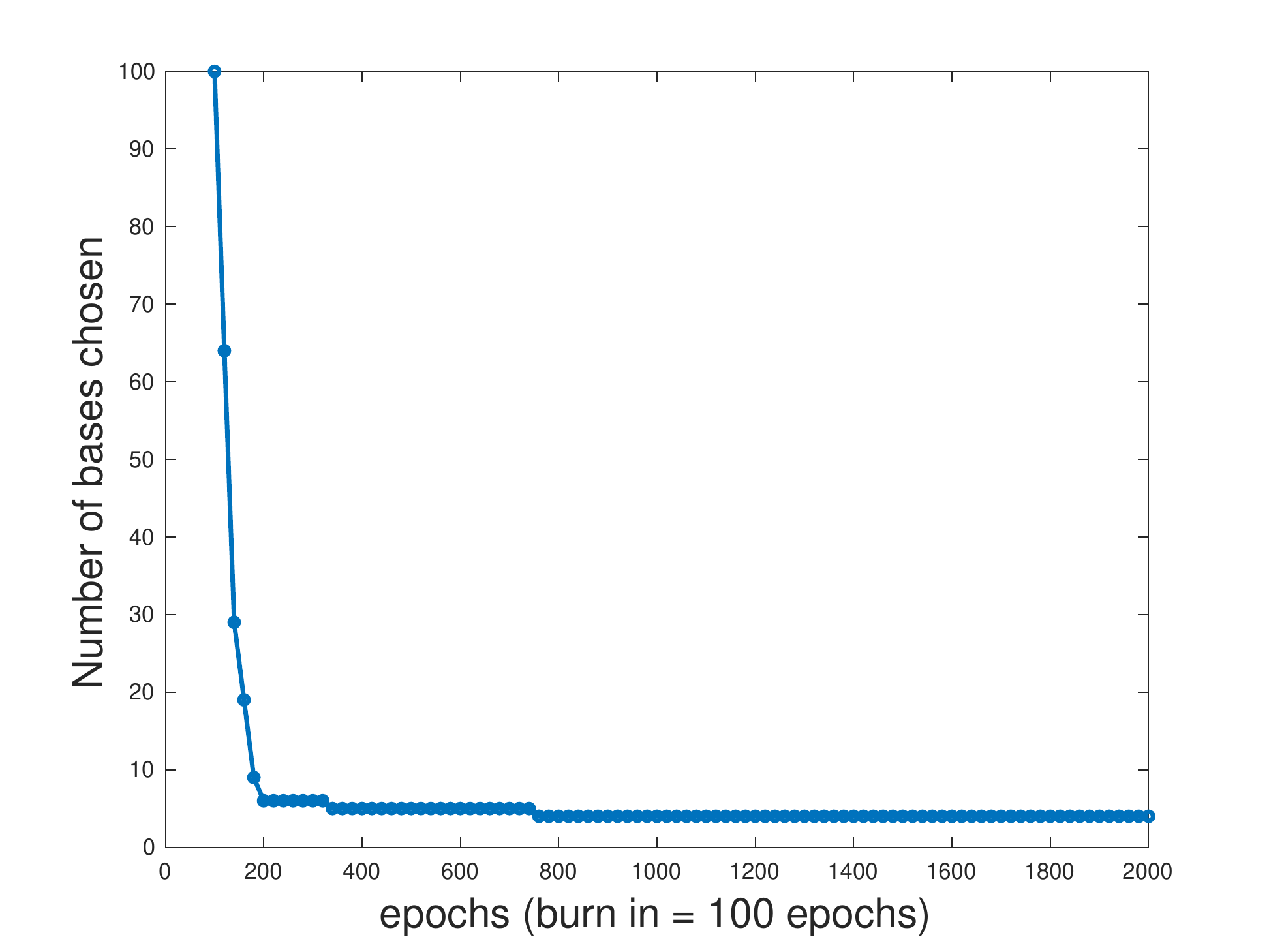}
	\caption{(Section \ref{sec:toy_nonlinear}) Toy example, nonlinear case, the number of selected bases v.s. training epochs. The number of training samples is $800$. The total number of basis functions is 100. After $200$ epochs, the sparsity reaches 90\%. The network identifies the correct sparsity pattern and chose the true basis functions in epoch $760$.}\label{fig:ex1_nonlinear_basis}
\end{figure}

The training and testing history are presented in \ref{fig:ex1_nonlinear}. We obtain similar results as shown in the linear case, where both the training and testing errors obtained from the proposed adaptive pruning method outperform the non-pruning case.

\begin{figure}[!hbt]
	\centering
	\includegraphics[scale=0.3]{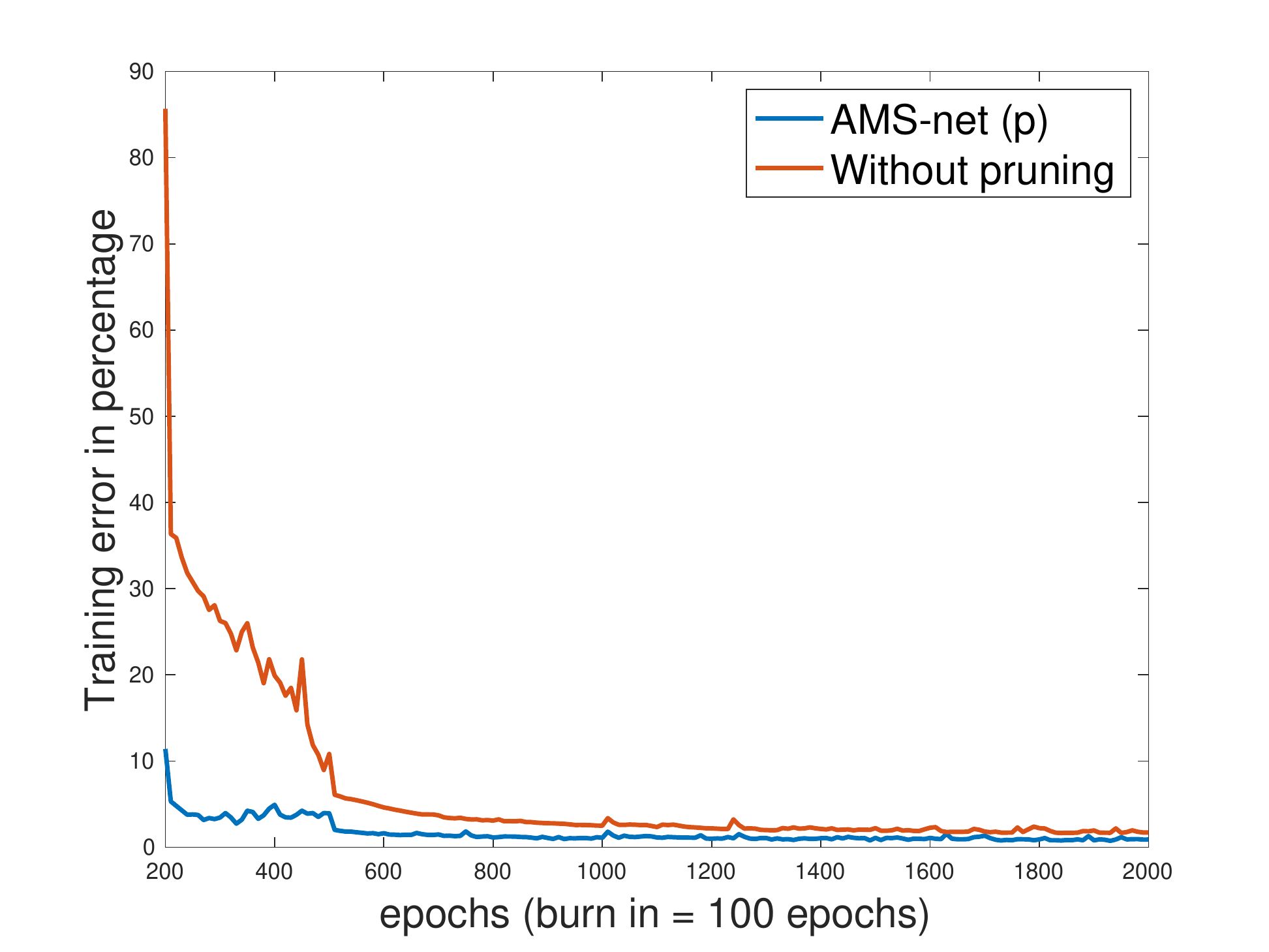}
	\includegraphics[scale=0.3]{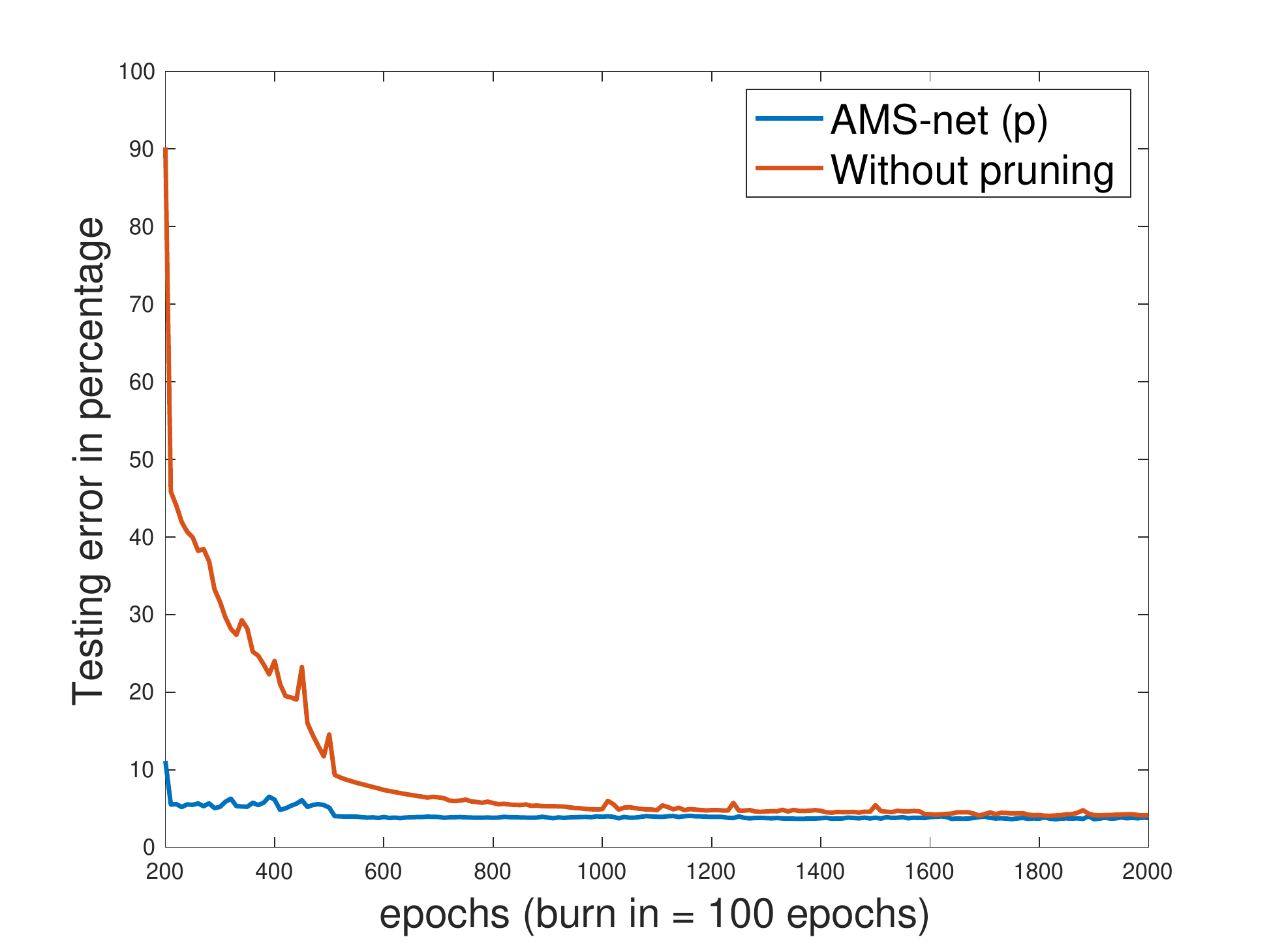}
	\caption{(Section \ref{sec:toy_nonlinear}) Toy example, nonlinear case. Left: training error history, right: testing error history. the number of training/testing samples is 800/200. Comparison between the results using our proposed pruning algorithm and without pruning. It shows that the sparse network training is more efficient.}\label{fig:ex1_nonlinear}
\end{figure}

\subsection{Flow dynamics}

\subsubsection{Velocity sparse approximation in multiscale space} \label{sec:vel_num}
We will generate samples by solving the system \eqref{eq:vel_2ph} and \eqref{eq:sat} sequentially on the fine grid with different source terms.
We take 
\begin{equation*}
    f_i(x,y) = \begin{cases}
      r_1 & \text{if $1-H<x<1 \; \& \; 0<y<H$}\\
      r_2& \text{if $0<x<H \; \& \;1-H<y<1 $}\\
            r_3 & \text{if $0<x<H \; \& \; 0<y<H$}\\
      r_3& \text{if $1-H<x<1   \; \& \;1-H<y<1  $}\\
        -(r_1+r_2+r_3+r_4)& \text{if $5H<x<6H \; \& \;5H<y<6H $}\\
      0 & \text{otherwise}
    \end{cases} 
\end{equation*}
where $r_i$ are randomly chosen in $[0,1]$, $i=1,\cdots,1500$. The absolute permeability $\kappa$ is set to be a layer in SPE10 model. An illustration of the source $f$ and $\kappa$ are shown in Figure \ref{fig:f_kappa}. The simulation is performed on the time interval $[0,16]$, with time step size $\triangle t = 4$. Thus, for each $f_i$, we have fine scale velocity solutions $[v_i^0, v_i^1, \cdots, v_i^2]$ at time steps $t=0, t=8, t=16$, respectively. Our goal is to use $f_i$ as input, $(v_i^0, v_i^1, \cdots, v_i^2)$ as labels to train the neural network using the loss function \eqref{eq:loss}. The dictionary $\mathcal{D}_{\text{vel}}$ \eqref{eq:dic_velocity} is constructed as described in section \ref{sec:vel_basis}. In our example, the computational area is $[0,1]\times [0,1]$,  the fine grid mesh size is $h=\frac{1}{50}$, and the coarse grid mesh size is $H=\frac{1}{10}$. There are $5400$ fine edges and $220$ coarse edges. For each interior coarse edge, we compute $5$ multiscale bases, resulting $900$ bases at each time instance. We note that we let $r_1=r_2=r_3=r_4=1$ in the source $f$ to generate offline basis functions, and only the bases obtained at time instances $t=0, t=8, t=16$ are included in the dictionary. Thus, we get $2700$ bases in the dictionary $\mathcal{D}$. Each input is a $100 \times 1$ vector which represents a coarse grid level source, and each $v_i^j$ ($j=0, \cdots, 3$) is a $5100\times 1$ vector. We define the error $e_1 = \frac{\norm{u_{\text{pred}} - u_{\text{true}} }_{L^2}}{\norm{u_{\text{true}} }_{L^2}}$.

\begin{figure}[!hbt]
	\centering
	\includegraphics[scale=0.3]{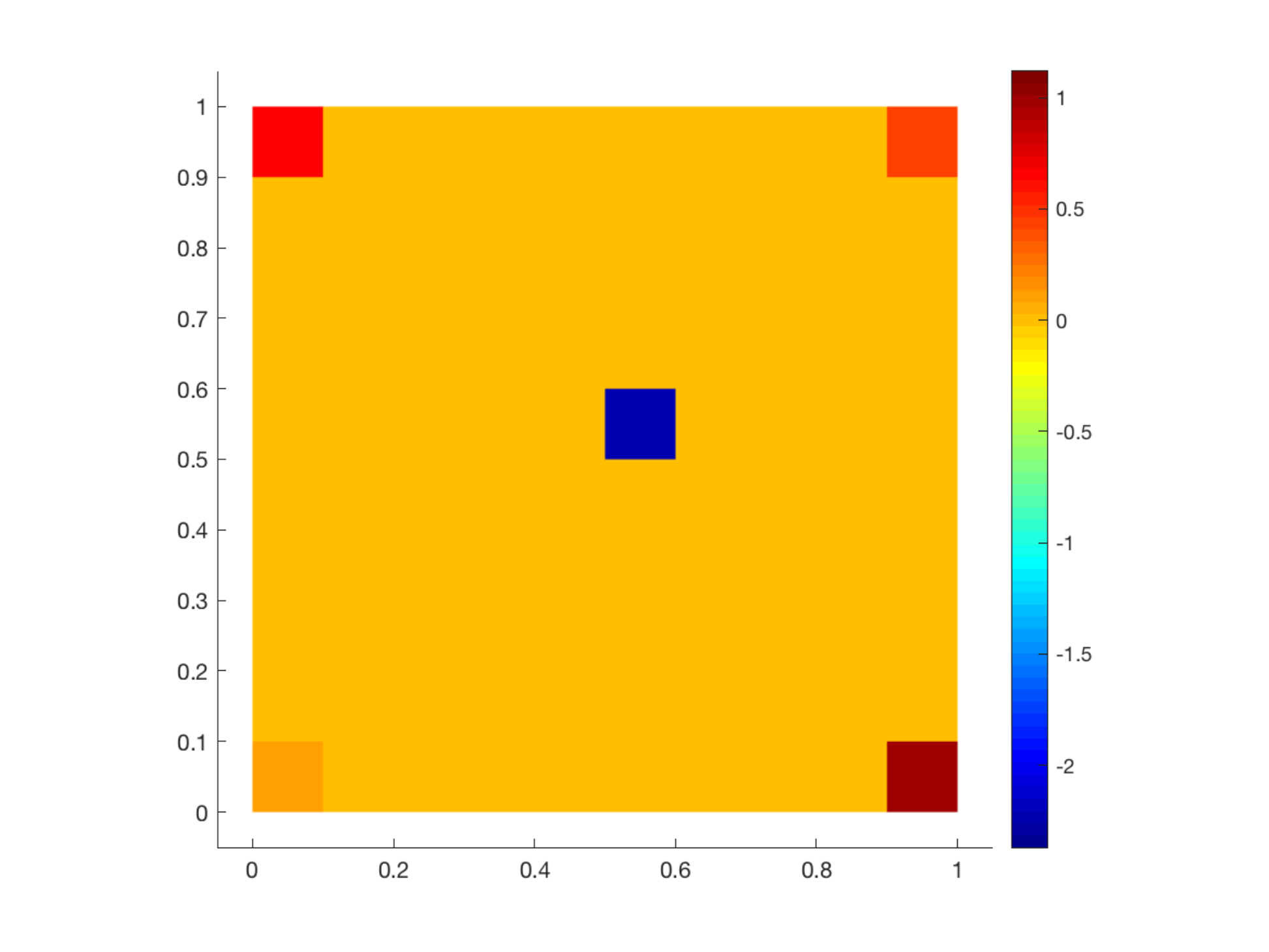} \includegraphics[scale=0.3]{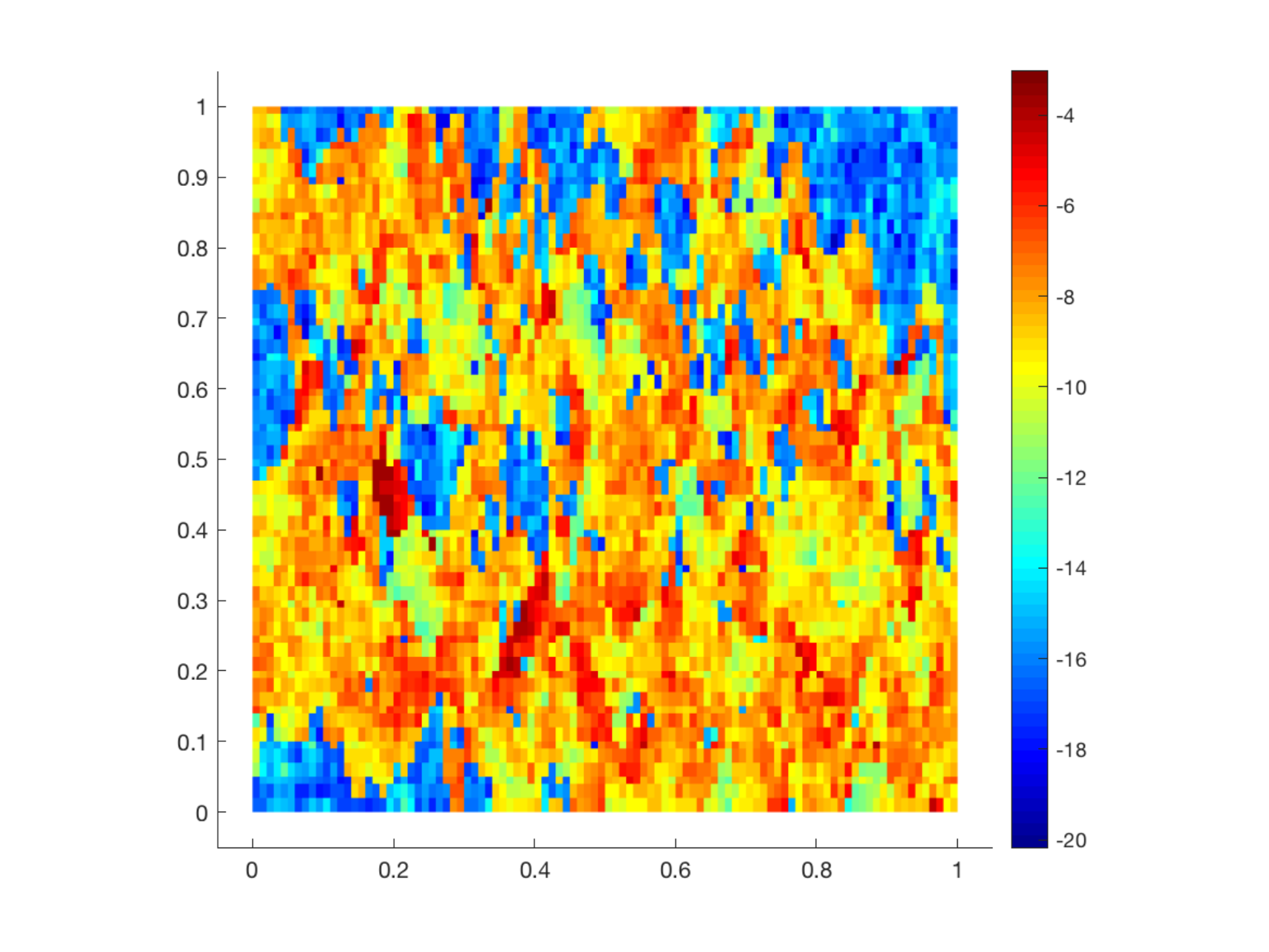}
	\caption{(Section \ref{sec:vel_num}) An illustration of the source term $f$ (left) and absolute permeability $\kappa$ in log scale (right). } \label{fig:f_kappa}
\end{figure}

\noindent We use $80\%$ of samples for training and the rest for testing. The total number of training epochs is 2000. We compare the results when we use the proposed algorithm \ref{alg:sparse}, with the results when we only apply pruning at the last epoch. 
 
 \noindent We use the notation `` AMS-net (p+a)" as a short for our proposed adaptive sparse method with both removing and adding bases,  `` AMS-net (p)" for our proposed adaptive sparse method with only removing bases. We also use the greedy algorithm to select bases based on the fine solution samples, and compute the projections of the fine scale solutions on the greedily selected space, and use them as a reference. 
 
 \noindent  We first set the target number of bases we would like to include in the training. The comparison of the errors obtained from these three approaches are presented in Table \ref{tab:example2}. The degrees of freedom (dof) in the table are the average of the degrees of freedom among solutions at all 6 time steps. We observe that with the similar size of the degrees of freedom, `` AMS-net (p+a)" outperforms `` AMS-net (p)". When the dof is less than 1000, `` AMS-net (p+a)" also produces better results compared with the greedy projection error. As the degrees of freedom increases, the mean prediction error decreases consistently.
 
 \noindent We also test the approach with a given threshold $\gamma_t^{\text{add}}$ to guide the procedure of adding bases, the results are presented in Figure \ref{fig:example2-2}. We note that $\gamma_t^{\text{add}} = \gamma_0^{\text{add}} \frac{100}{(100+t)^{0.75}}$ is a decreasing sequence. It shows that as $\gamma_0^{\text{add}}$ decreases, the algorithm selects more degrees of freedom and the error of prediction decreases too. The decreasing rate is almost linear. This demonstrates that with the threshold $\gamma_t^{\text{add}}$, we can control the training error of the network and thus obtain better prediction results. When $\gamma_t^{\text{add}}$ approach 0, the mean prediction error is $0.95\%$, which is some irreducible snapshot error.

\noindent We also present the prediction errors using the adaptive sparse method with both pruning and adding bases at each time step in Figure \ref{fig:ex2_allsteps}. Two random test cases are shown in Figure \ref{fig:ex2_sample1} and \ref{fig:ex2_sample2}, where we see good matches between the reference solutions and network predictions.

\textit{Remark}: In this work, we assume the dictionary can be built in advance. The basis functions in the dictionary are found by solving some local problems corresponding to various permeabilities based on developed numerical techniques. Once computed, the dictionary won’t be reconstructed when the permeability field changes. Given this large dictionary, one may first employ some preselection techniques such as clustering to narrow down the search of basis functions for a given application. Then our proposed method will be beneficial to find a much sparse representation for the quantities of interest. The focus of our paper is to develop an efficient and stable algorithm to obtain a fast and computational cheap solver given the dictionary. After the sparse network is trained given the data, we can apply it to evaluate the test cases in a very fast manner.


\begin{table}[!htb]
	\centering
	\begin{tabular}{| c |c  |c | c  |}
		\hline
		dofs (approximate) & AMS-net (p) & AMS-net (p+a) & Greedy projection error \\  \hline
	    400  &9.37 &7.23 &7.51  \\  \hline
		500 &7.00 &4.78 &5.00  \\  \hline
		600  &5.48 &3.05 &3.84 \\  \hline
		700 &4.29 &1.88 &2.82   \\  \hline
		800  &3.04 &1.23 &2.10   \\  \hline
		900  &2.31 &1.07 &1.44   \\  \hline
		1000  &1.78 &0.95 & 0.94  \\  \hline
		
	\end{tabular}
	\caption{(Section \ref{sec:vel_num}) Learning velocity fields. Mean errors between the true and predicted velocity solution among $200$ testing samples. Column 2 shows the results when we just use pruning strategy without adding basis. Column 3 shows the results when we use the adaptive strategy with both pruning and adding basis. In column 4, we use the greedy algorithm to select bases based on the training samples, and compute the mean of projection error for the testing cases in the greedily selected space. The errors are in percentage.}\label{tab:example2}
\end{table}

\begin{figure}[!hbt]
	\centering
	\includegraphics[scale=0.3]{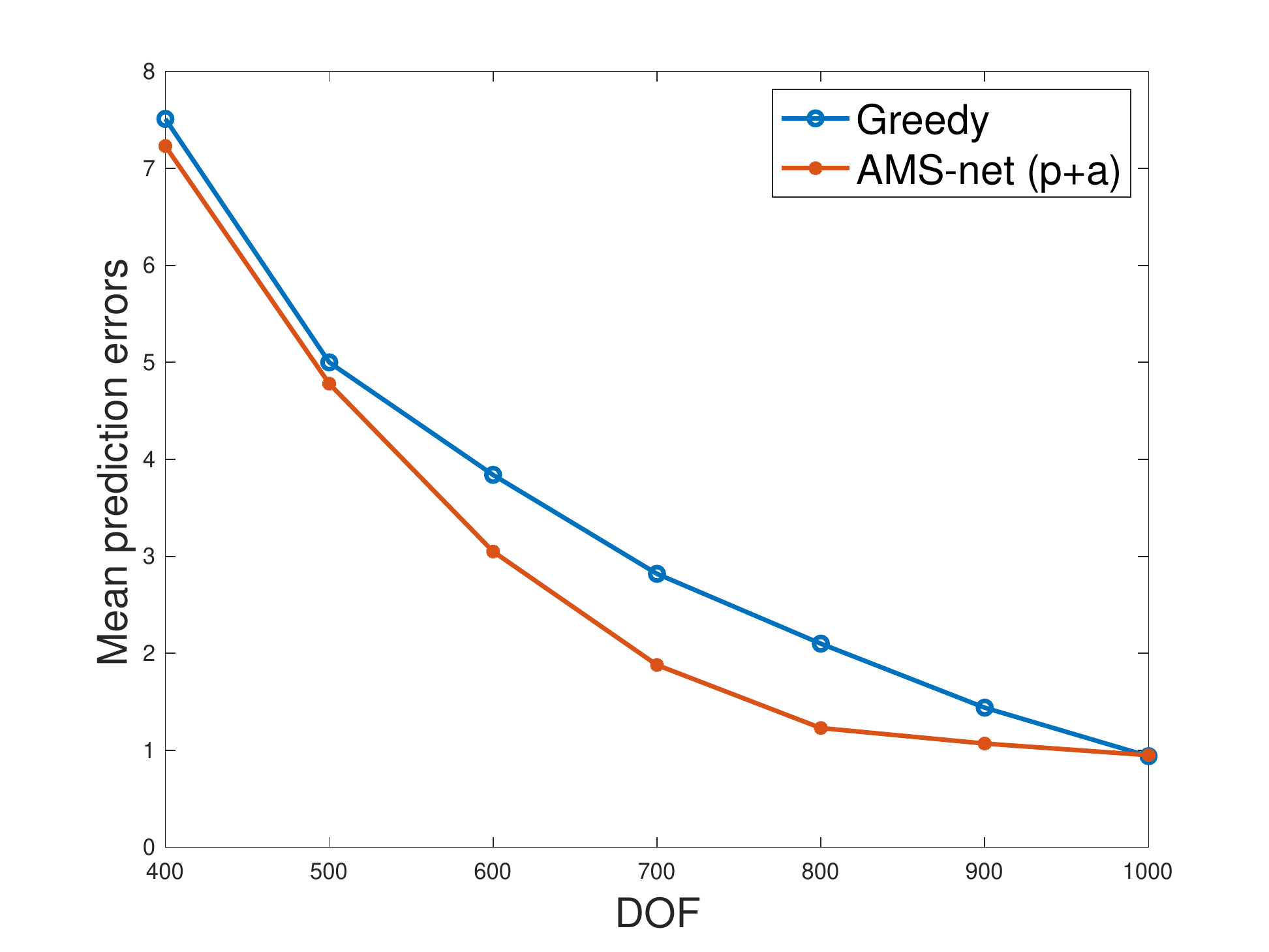}
	\includegraphics[scale=0.3]{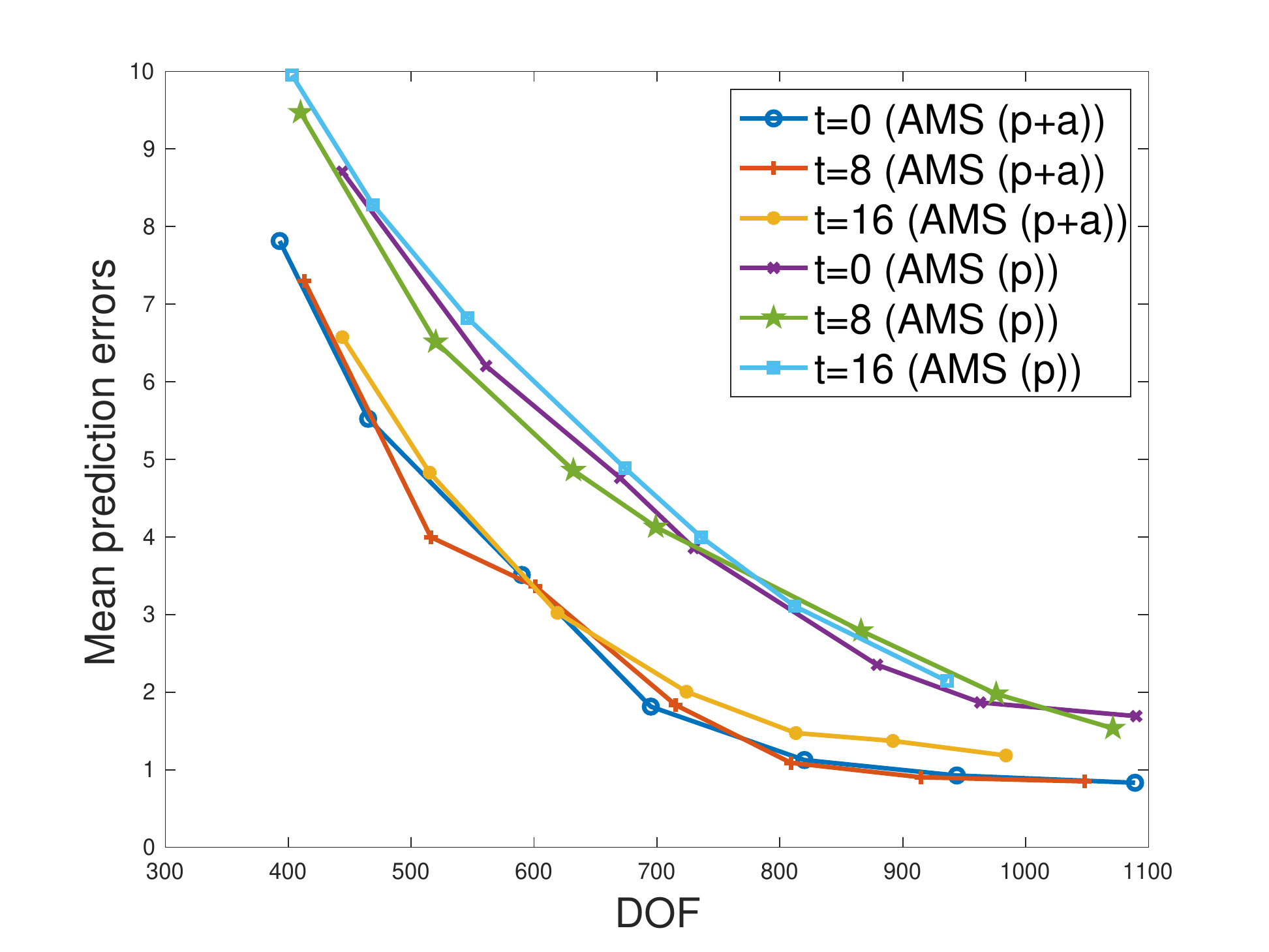}
	\caption{(Section \ref{sec:vel_num}) Learning velocity fields. Mean prediction errors among 200 testing samples, at different time steps and different dofs. With a increasing number of basis selected in AMS-net, the prediction errors will decrease. Left: mean errors over all time steps. Our algorithm produces better results consistently when the dof is less than 1000, and it converges to the snapshot error when the dof become larger. Right: mean errors at each time step.
	} \label{fig:ex2_allsteps}
\end{figure}


\begin{figure}[!hbt]
	\centering
	\includegraphics[scale=0.4]{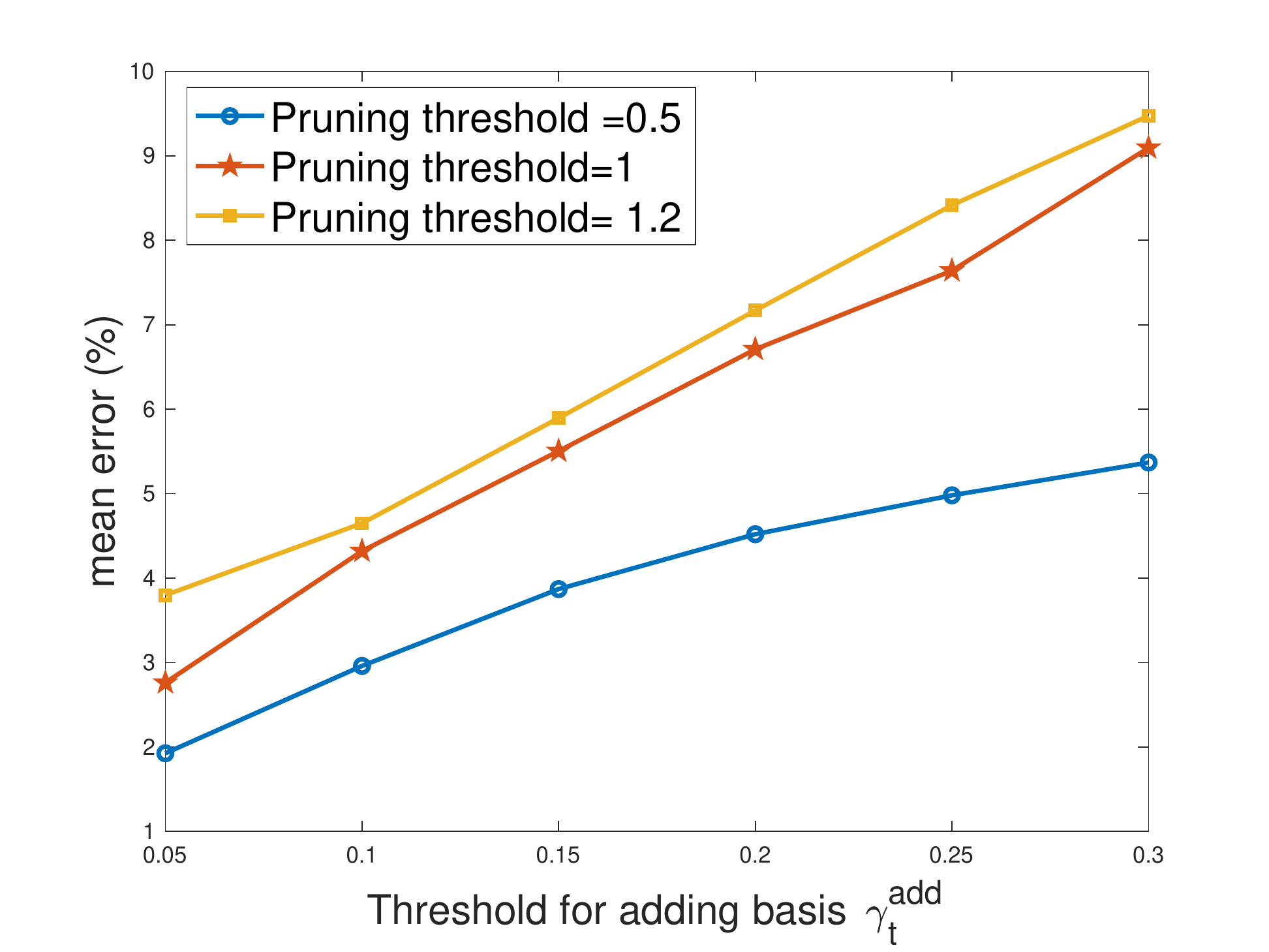}
	\caption{(Section \ref{sec:vel_num}) Learning velocity fields. Using given thresholding parameters $\gamma_t^{\text{add}}$ to adaptively add bases in AMS-net. Mean errors between the true and predicted velocity among $200$ testing samples. It shows that decreasing the value of thresholding parameter can help to control the accuracy of AMS-net predictions. }\label{fig:example2-2}
\end{figure}

\subsubsection{Saturation sparse approximation in POD space} \label{sec:sat_num}
In this example, we take similar source configurations as described in the previous section. 
We first solve the system with a specific source term on the time interval $[0,4]$ with time step size $\triangle t = 0.1$. Gather the saturation solutions on these $40$ time instances together, we will perform POD on it and select the resulting bases to form $V_{\text{sat}}$, and $\mathcal{D}_{\text{sat}}$ in \eqref{eq:dic_saturation}.

\noindent Next, for each source $f_i$ ($i = 1,\cdots, 1000$), we have solve for fine saturation solutions on the time interval $[0,6]$ with time step size $\triangle t = 0.1$. However, we only take the solutions every $10$ time steps to train the neural network, i.e, $[S_i^1, S_i^{10}, \cdots, S_i^{60}]$. Again, we use $f_i$ as input, $(S_i^1, S_i^{10}, \cdots, S_i^{60})$ as labels, and the previously mentioned dictionary to train the neural network. In this case, the fine degree of freedom for the saturation solution is $10000$, and the reduced order space has dimension $40$. We define the error $e_1 = \frac{\norm{S_{\text{pred}} - S_{\text{true}} }_{L^2}}{\norm{S_{\text{true}} }_{L^2}}$. The total number of training epochs is $2500$. 

\noindent We compare the results when we use the proposed algorithm \ref{alg:sparse} with both pruning and adding bases, the algorithm with only pruning and the POD projection error. The errors using these three approaches are listed in Table \ref{tab:example3}. The dofs in the first column of the table are the mean dofs at all 6 time steps. We observe that, AMS-net (with both pruning and adding bases) produces similar results when dof is equal to 40, and achieve better predictions when the dof is equal to 6, 8, 10, 12. It actually converges to the snapshot error when all bases are used (dof = 40).
A random test case is shown in Figure \ref{fig:example3}, where we see good matches between the reference and network prediction.


\begin{table}[!htb]
	\centering
	\begin{tabular}{| c |c  |c  | c |}
		\hline
    dofs  & AMS-net (p) & AMS-net (p+a) & POD projection \\  \hline
    6  &2.21   &1.96  & 2.10 \\  \hline
    8  &1.72  &1.62 &1.67 \\  \hline
    10  &1.45   & 1.35  &1.40 \\  \hline
    12 &1.23 &1.18& 1.22  \\  \hline
     40 &0.95   &0.95  &0.95 \\  \hline
	\end{tabular}
	\caption{(Section \ref{sec:sat_num}) Learning saturation profile. Fix the number of bases in the adaptive process. Mean errors between the true and predicted saturation among $200$ testing samples. AMS-net (with both pruning and adding bases) produces better results when the dof is equal to 6, 8, 10, 12, and it converges to the projection error when all the bases are used (dof = 40). The errors are in percentage.}\label{tab:example3}
\end{table}

\begin{figure}[!hbt]
	\centering
	\includegraphics[scale=0.28]{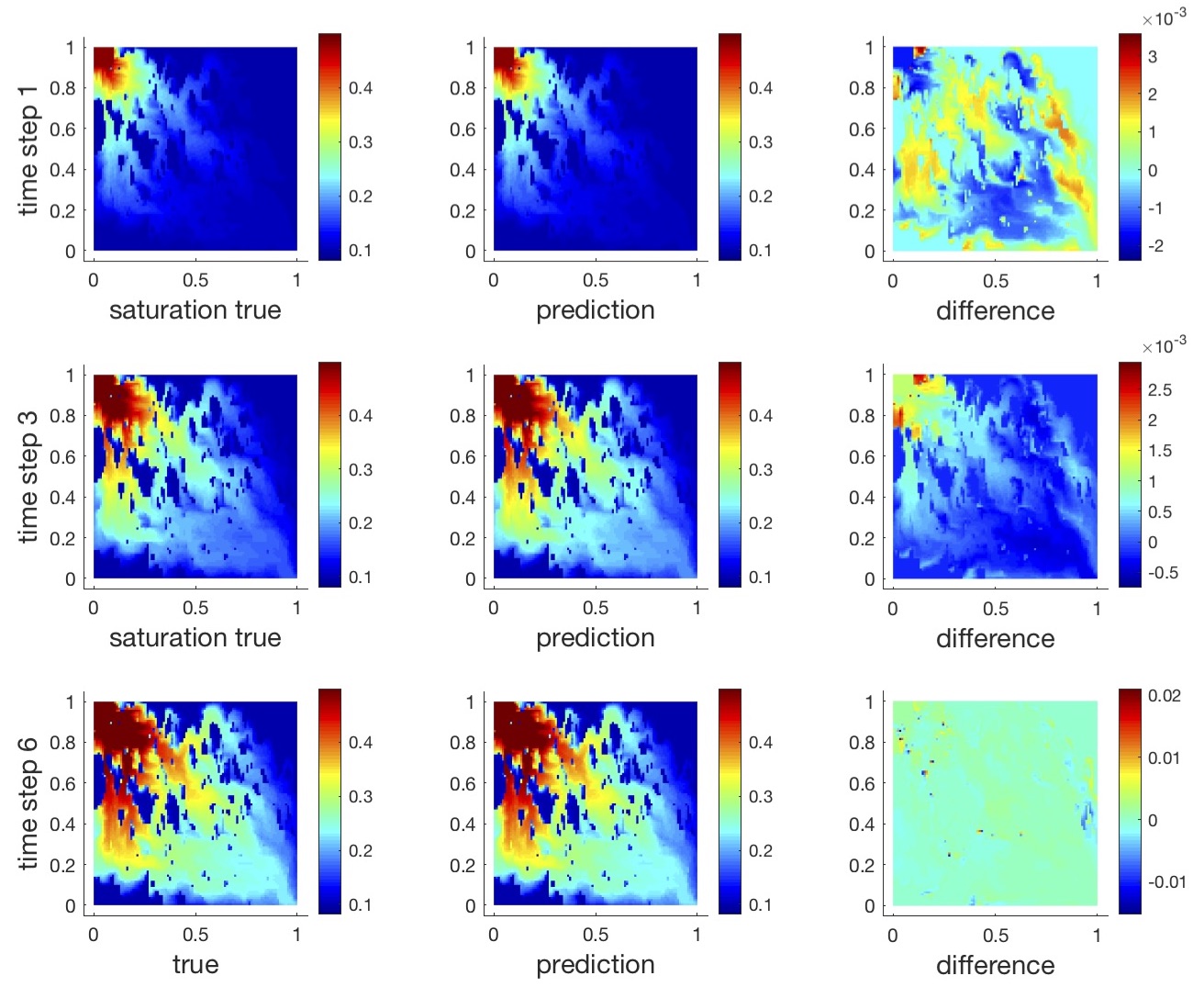}
	\caption{(Section \ref{sec:sat_num}) Learning saturation profile. Test case illustration: saturation at all the first, third and last time steps. The predicted results matches the reference solution well.} \label{fig:example3}
\end{figure}

\textit{Remark}: Our method can be extended naturally to practical applications with more complicated physics, such as permeability variations, compressible fluids or the case with gravitational effects. First of all, we assume the dictionary is precomputed and contains enough basis functions which can capture the features of the underlying media for a wide range of cases. The construction of basis functions is discussed in many existing works and is beyond the scope of this paper. The key of our method is to adaptively select important basis functions from the known large dictionary through the proposed sparse learning algorithm and obtain reduced order models for the dynamical system. 

In the case of compressible fluids, for example, the two phase flow with gas and oil, the basis functions we need might be different at different time steps. If the gravity is also considered, extra bases are needed to approximate the gravity force. In such cases, our method can automatically choose suitable sets of basis from the dictionary at corresponding time instances and capture the complex physical properties. Since the dictionary can be very large to cover complicated applications, our method is more beneficial to obtain the sparse representation and achieve an efficient approximation. For the more complex nonlinear system, one may also consider employing a larger neural network to approximate the dynamics, this will cause more difficulties for the learning process. Our algorithm enforces sparsity in the network connections and can help training.

\subsubsection{Saturation sparse approximation: a simple illustration of interpretibility}\label{sec:sat_interp}

In this section, we will use simple basis functions to approximate the saturation solution to illustrate the interpretibility of the proposed method. 
Consider the 10-by-10 coarse mesh in the domain $[0,1]\times [0,1]$, the absolute permeability is a fractured media and the source term $f$ has a two-spot well configuration, as shown in Figure \ref{fig:ex3_sat_mesh_source}. To simplify, we will use  piecewise constant basis functions. That is, for each coarse block and each fracture segment inside the coarse block, we will have a degree of freedom associated with it. Each basis function has a value equal to $1$ in the coarse block or the fracture segment, and has value equal to $0$ elsewhere. The degrees of freedom (dof) with respect to fracture segments are labeled from dof1 to dof 21 as presented in Figure \ref{fig:ex3_sat_mesh_source}. We perform simulations on the time interval $T=(0,60]$, with time step size $\triangle t = 1$, and only select the solutions at time instances $t=10, 20, 30, 40, 50, 60$ to train the neural network with the proposed method. Let $r=2$ be the injection rate in the source term, the solutions at time steps $t=10,30,60$ are shown in Figure \ref{fig:ex3_sat_solution}. We observe that the saturation 
hardly goes into the fracture associated with dof 2, dof 4, dof 6, and dof 8. Moreover, at the early time steps, the fluid did not saturate into the fracture associated with dof 15 - dof 21, but was fully saturated in the last time step. Now, we choose a set of $100$ random injection rates in $[1.5, 2.5]$, and use their corresponding solutions to train the network. Our purpose is to observe how the network chooses basis functions to represent fractures and matrices.
The results are shown in Table \ref{tab:ex3_dofs}. We see that the network can identify important dofs correctly. This shows the potential of interpretibility of our proposed method.

\begin{figure}[!htp]
	\begin{center}

			\includegraphics[scale=0.2]{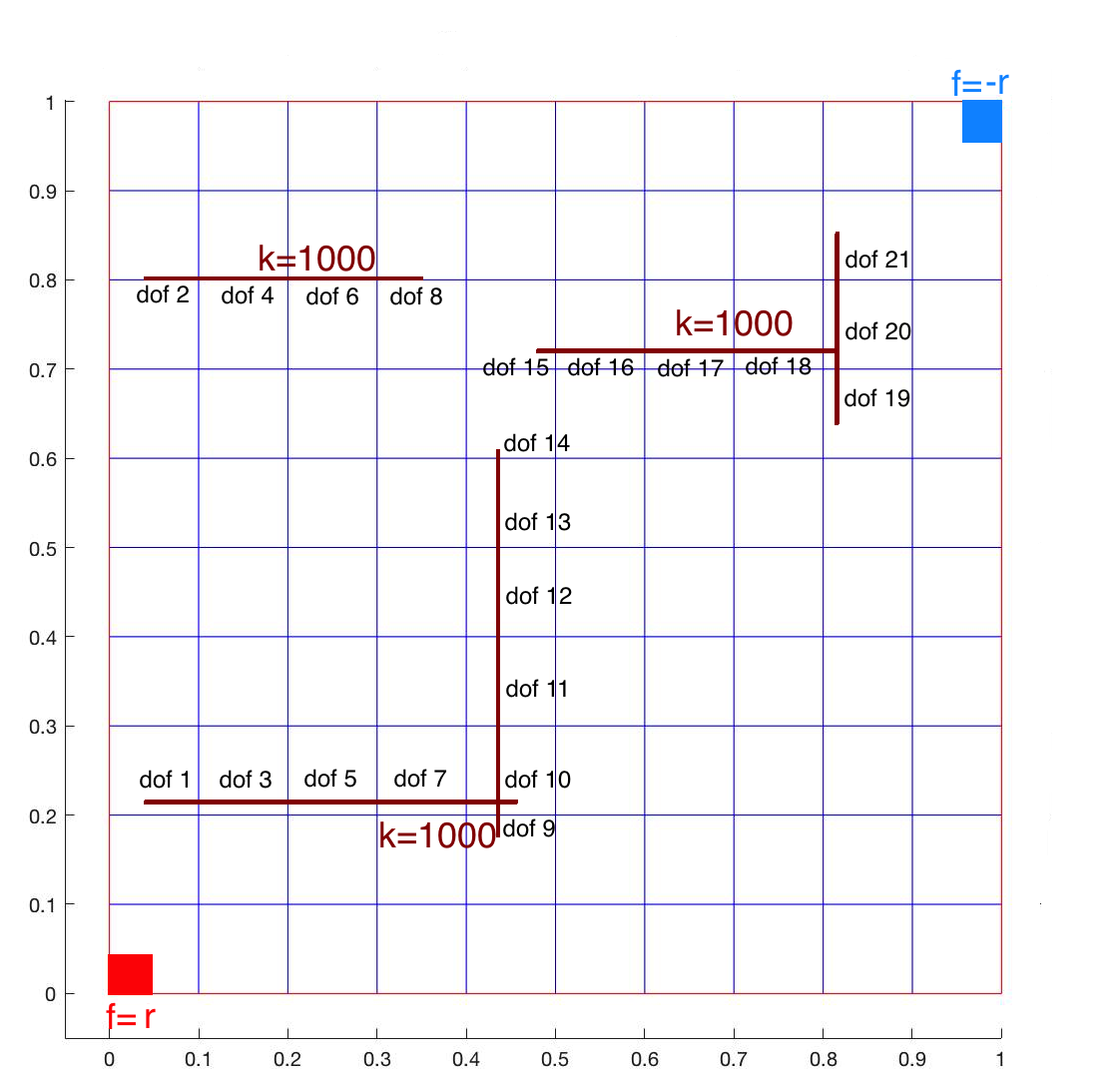}
			\caption{(Section \ref{sec:sat_interp}) The background shows the 10-by-10 coarse mesh. The absolute permeability $\kappa$ takes value $1$ in the background, and takes value $1000$ in the maroon-colored channels. The source function $f$ takes value $0$in the background, $f=r (r>0)$ in the red region in the bottom left corner, and $f=-r$ in the red region in the top right corner. The degrees of freedom (dof) with respect to channels are labeled from dof1 to dof 21.}  \label{fig:ex3_sat_mesh_source}
	\end{center}
\end{figure}

\begin{figure}[!htp]
	\centering
			\includegraphics[scale=0.2]{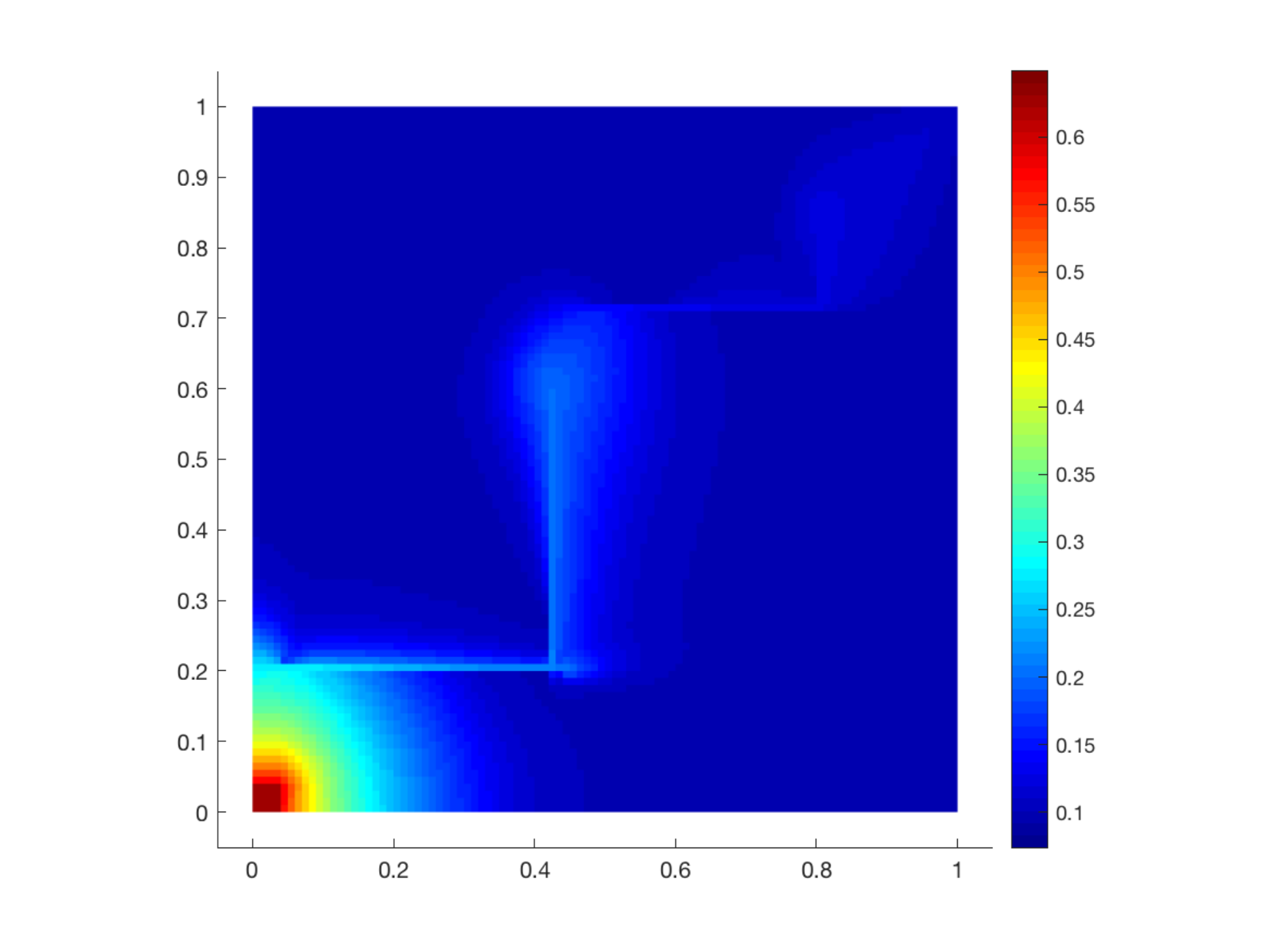}
			\includegraphics[scale=0.2]{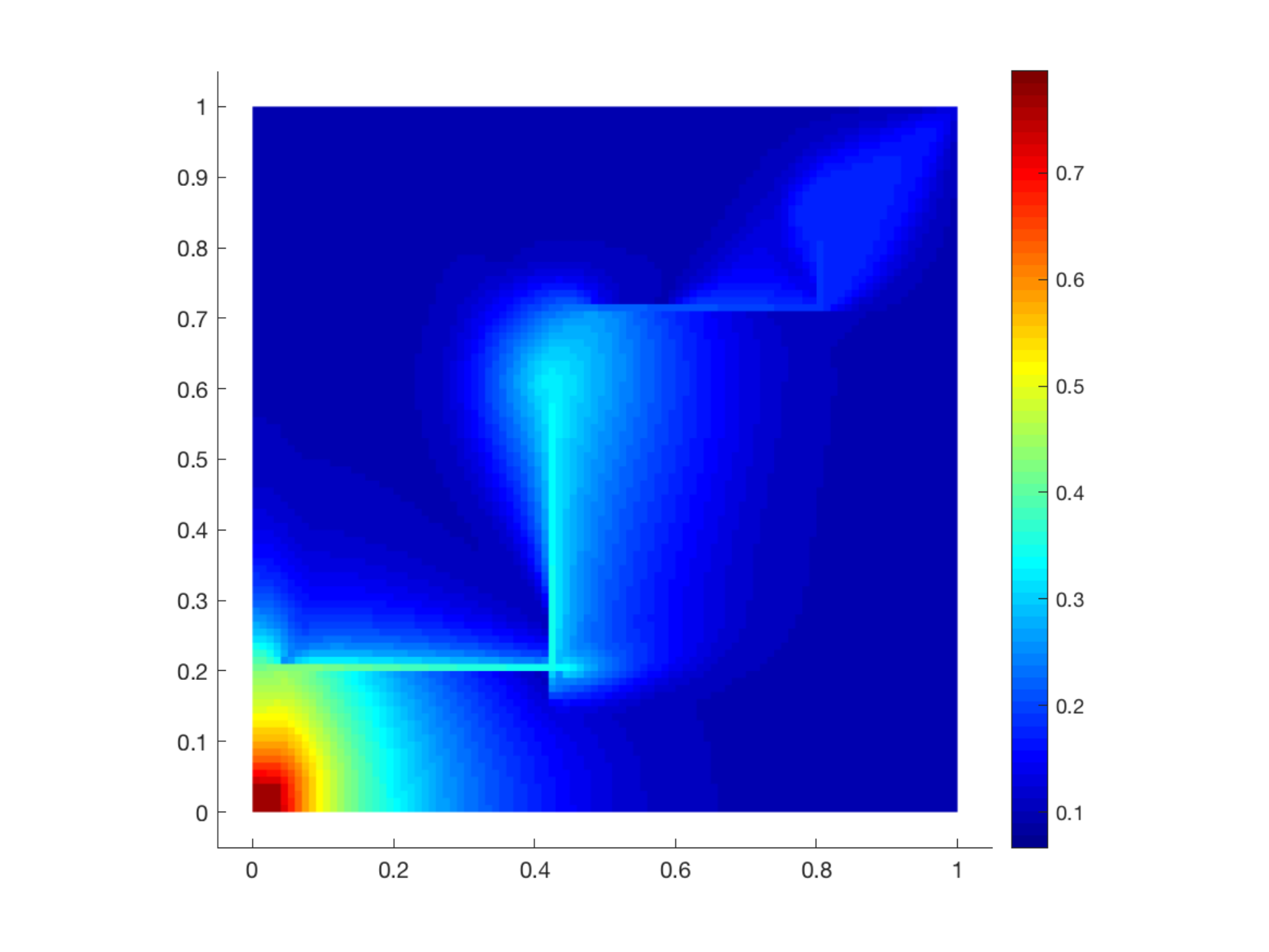}
			\includegraphics[scale=0.2]{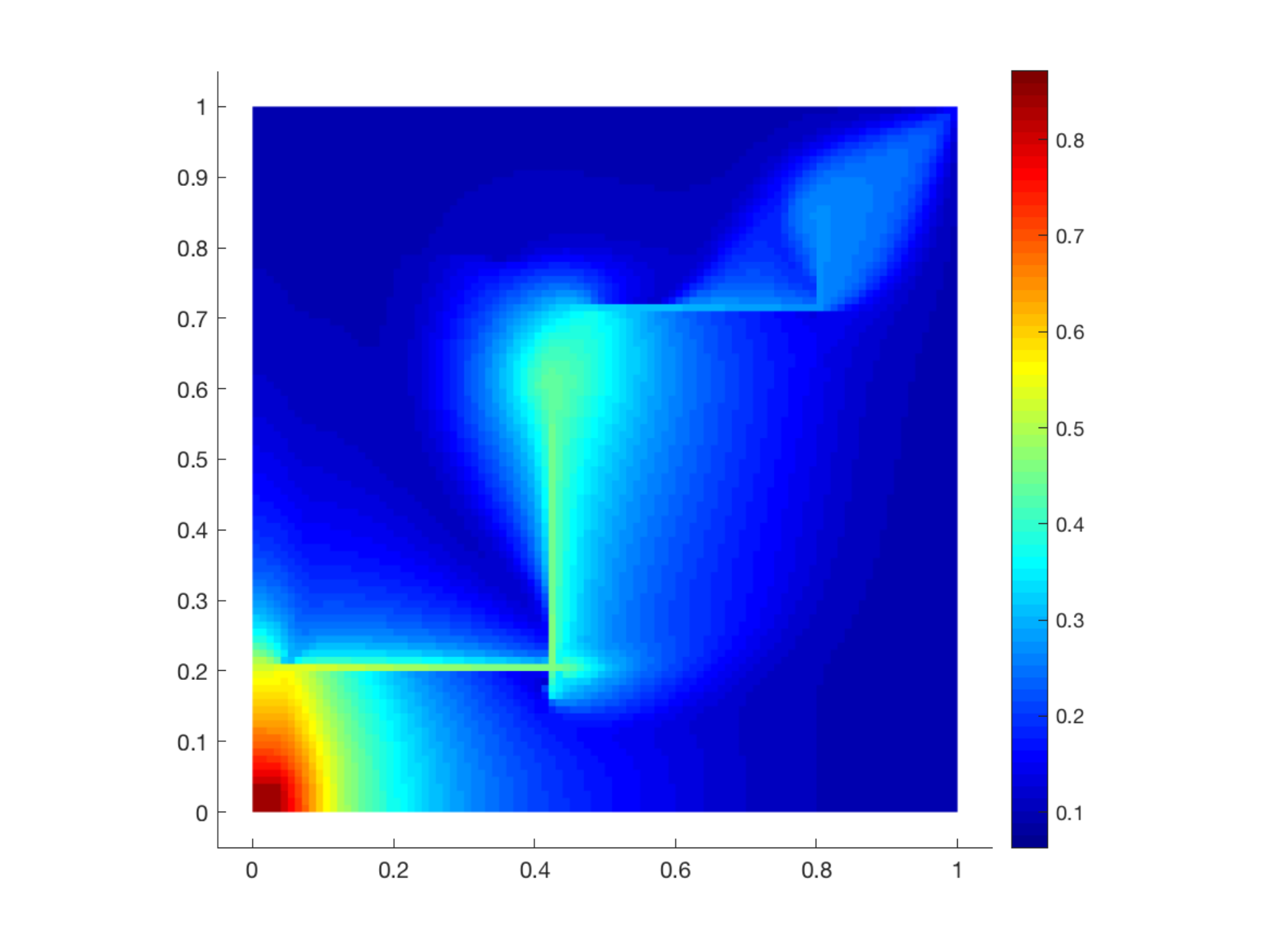}
			\caption{(Section \ref{sec:sat_interp}) From left to right: saturation at $t=10$, $t=30$, and $t=60$.}\label{fig:ex3_sat_solution}
\end{figure}

\begin{table}[!htb]
	\centering
	\begin{tabular}{| c |l ||}
		\hline
		Time step& The unselected dofs resulted from AMS-net  \\  \hline
	    t=10  &  2, 4, 6, 8, 9, 13, 14, 15, 16, 17, 18, 19, 20, 21  \\  \hline
		t=20 & 2, 4, 6, 8,  13, 14, 18, 19, 20, 21   \\  \hline
		t=30 & 2, 4, 6, 8,   14, 19, 20, 21 \\  \hline
		t=40 & 2, 4, 6,  14, 19, 21  \\  \hline
		t=50  & 2, 4, 6, 19, 21   \\  \hline
		t=60 &  2, 4, 19, 21  \\  \hline
	\end{tabular}
	\caption{(Section \ref{sec:sat_interp}) The unselected dofs resulted from AMS-net at all time steps. The saturation hardly goes into the fracture associated with dof 2, dof 4, dof 6 and dof 8. Moreover, the fluid didn't saturated into the fracture associated with dof 15 - dof 21 at the early time steps, but was fully saturated in the last time step. The desired dofs are identified by our algorithm.}\label{tab:ex3_dofs}
\end{table}

\section{Conclusion}\label{sec:conclusion}
We present a scalable sparse learning framework, which incorporates some precomputed basis functions in the learning objective. The network aims to learn the flow dynamics where the parameters in the flow model contain multiscale properties. The inputs are random source terms, and the labels are fine scale solutions at different time steps. The outputs of the neural network are coefficients of the solutions corresponding to the basis functions at these time instances. The predicted solutions are then formed by the product of the coefficient vectors and basis functions. The objective is to minimize the differences between the predicted solutions and fine scale reference solutions over all time steps. The algorithm can adaptively choose important basis functions from a large pool of different source inputs. The sparsity in the solution coefficient vector is enforced through a built-in thresholding operator, which is implemented as an activation function in some layers of the network. The sparsity of layer connections in the network is achieved by cutting the connections to coefficients with small magnitude. To avoid dropping too many basis functions and enrich the approximation space during the training, one can also add some degrees of freedom back through a greedy procedure. Through adaptive training, one can obtain a sparse set of important basis functions and an accurate approximation to the flow dynamics. Several numerical tests are performed to demonstrate sparse and accurate approximations to the solutions using the proposed algorithm.  

\section*{Appendix}

\begin{figure}[!hbt]
	\centering
    \includegraphics[scale=0.42]{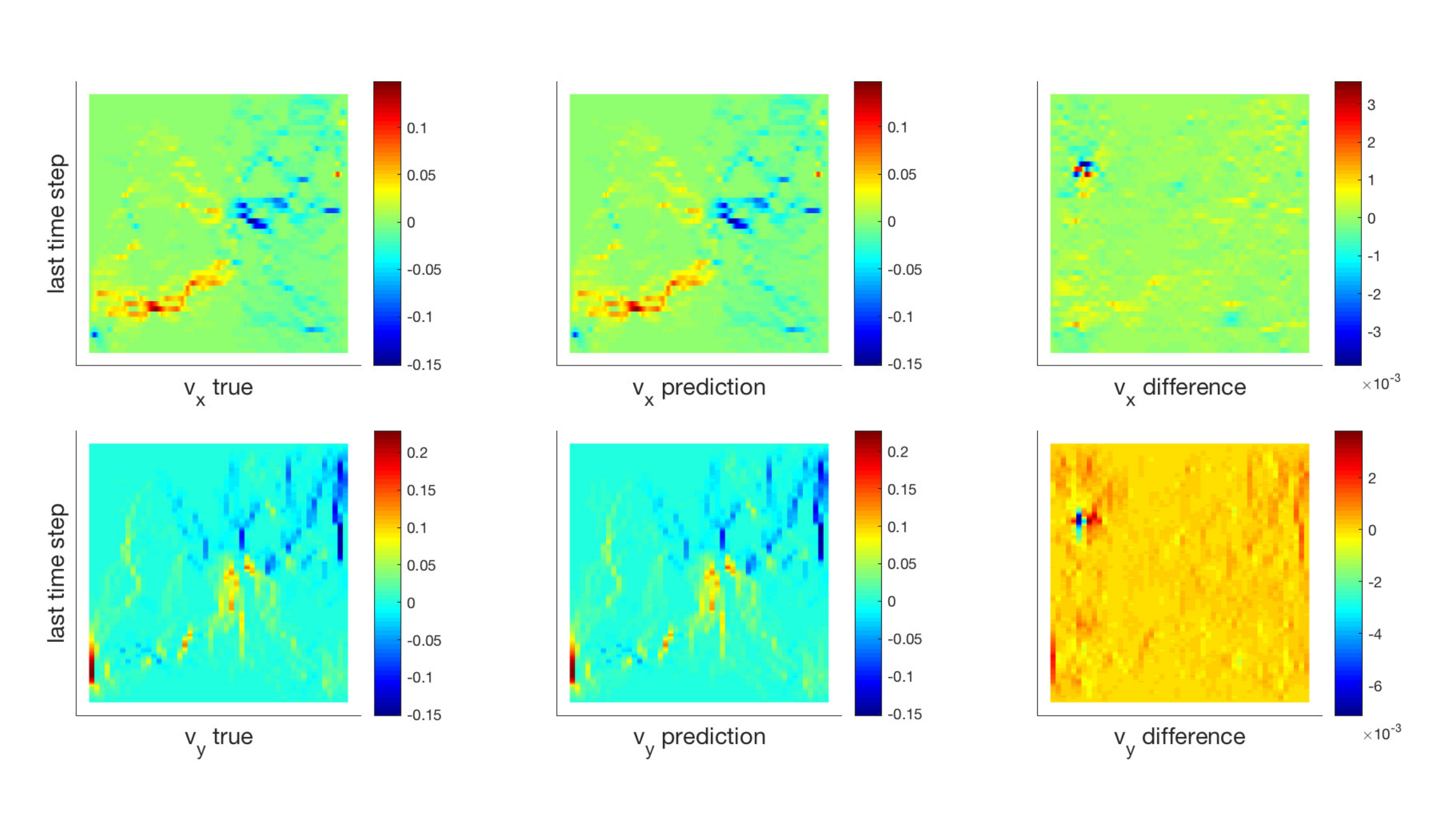}
	\caption{(Section \ref{sec:vel_num}) Learning velocity fields, test case 1. AMS-net predictions produces accurate predictions using 1000 basis. The relative $l_2$ error is 1.58\% at the last time step. }\label{fig:ex2_sample1}
\end{figure}

\begin{figure}[!hbt]
	\centering
	\includegraphics[scale=0.42]{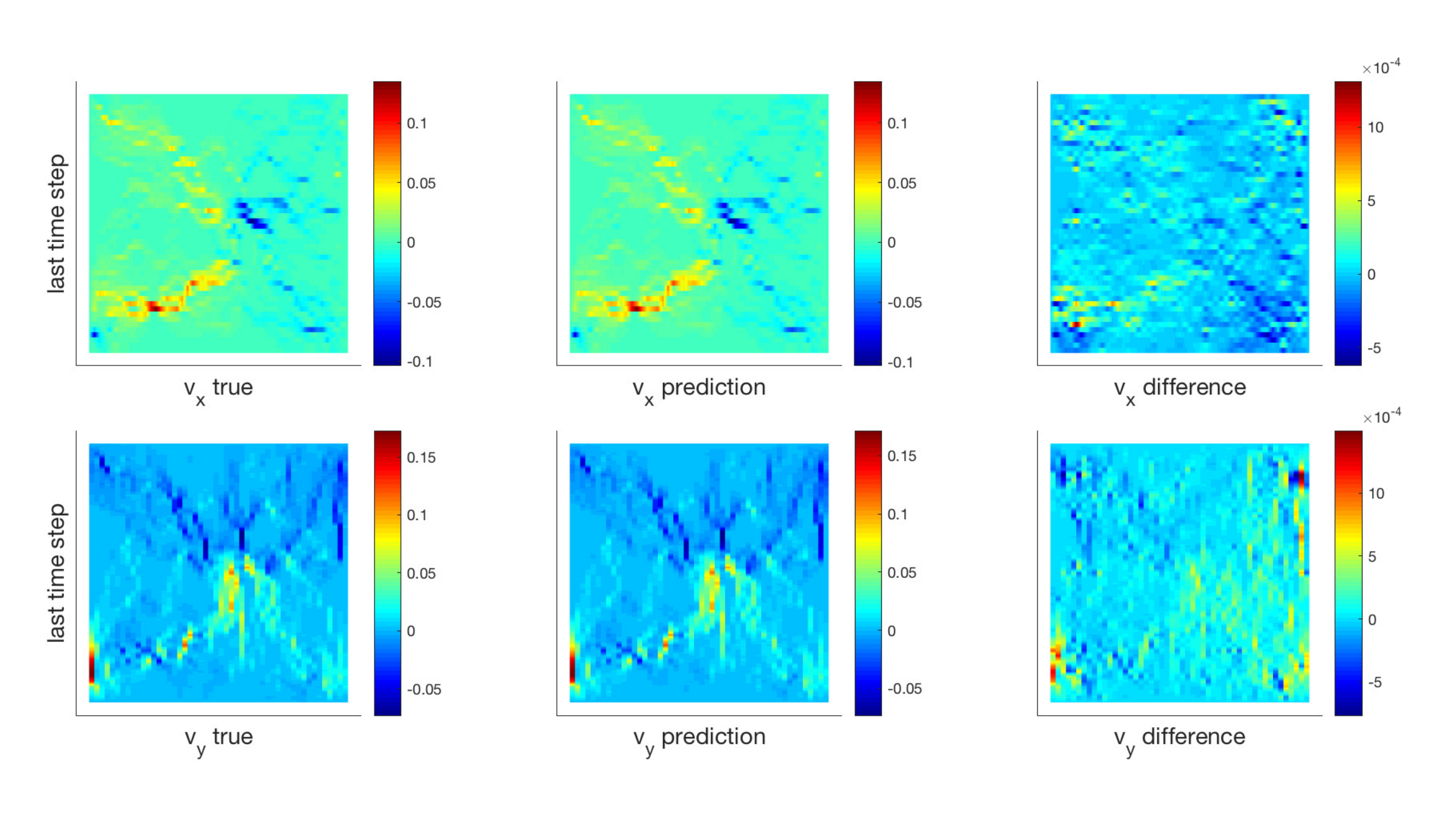}
	\caption{(Section \ref{sec:vel_num}) Learning velocity fields, test case 2. AMS-net predictions produces accurate predictions using 1000 basis. The relative $l_2$ error is 0.99\% at the last time step. }\label{fig:ex2_sample2}
\end{figure}


\section*{Acknowledgements}
We gratefully acknowledge the support from the National Science Foundation (DMS-1555072, DMS-1736364, DMS-2053746, and DMS-2134209), and Brookhaven National Laboratory Subcontract 382247, ARO/MURI grant W911NF-15-1-0562, and U.S. Department of Energy (DOE) Office of Science Advanced Scientific Computing Research program DE-SC0021142.

\bibliographystyle{siam} 
\bibliography{references}

\end{document}